\documentclass[12pt,letterpaper]{article}
\usepackage{amssymb,amsmath}
\usepackage{mathtools} 
\usepackage{amsthm}
\usepackage{hyperref}
\usepackage{graphicx}
\usepackage[table]{xcolor}
\usepackage{natbib}
\usepackage{ulem}
\usepackage{bbm} 
\usepackage{algorithm2e}
\usepackage{url} 

\DeclarePairedDelimiter\abs{\lvert}{\rvert}

\newcommand{\bX}{\boldsymbol{X}}

\newcommand\maT{\mathcal{T}}
\newcommand\maE{\mathbb{E}}
\newcommand\maR{\mathbb{R}}

\newtheorem{theorem}{Theorem}
\newtheorem{lemma}[theorem]{Lemma} 
\newtheorem{proposition}[theorem]{Proposition} 
\newtheorem{remark}[theorem]{Remark}
\newtheorem{corollary}[theorem]{Corollary}

\begin{document}

\title{Fast Linear Model Trees by PILOT}

\author{Jakob Raymaekers\\
  {\normalsize Department of Quantitative 
	Economics, Maastricht University, 
	The Netherlands}\\ \\
	Peter J. Rousseeuw\\
	{\normalsize Section of Statistics and
	Data Science,  
	University of Leuven, Belgium}\\ \\
	Tim Verdonck\\
	{\normalsize Department of Mathematics,
	 University of Antwerp - imec, Belgium}\\ \\
	Ruicong Yao\\
	{\normalsize Section of Statistics and
	Data Science,  
	University of Leuven, Belgium}\\ \\}
\date{\normalsize{February 7, 2023}}
\maketitle

\begin{abstract}
Linear model trees are regression trees that incorporate linear models in the leaf nodes. This preserves the intuitive interpretation of decision trees and at the same time enables them to better capture linear relationships, which is hard for standard decision trees. But most existing methods for fitting linear model trees are time consuming and therefore not scalable to large data sets. In addition, they are more prone to overfitting and extrapolation issues than standard regression trees. In this paper we introduce PILOT, a new algorithm for linear model trees that is fast, regularized, stable and interpretable. PILOT trains in a greedy fashion like classic regression trees, but incorporates an $L^2$ boosting approach and a model selection rule for fitting linear models in the nodes. The abbreviation PILOT stands for \textbf{PI}ecewise \textbf{L}inear \textbf{O}rganic \textbf{T}ree, where `organic' refers to the fact that no pruning is carried out. PILOT has the same low time and space complexity as CART without its pruning. An empirical study indicates that PILOT tends to outperform standard decision trees and other linear model trees on a variety of data sets. Moreover, we prove its consistency in an additive model setting under weak assumptions. When the data is generated by a linear model, the convergence rate is polynomial.
\end{abstract}

\noindent {\it Keywords:} Consistency, 
Piecewise linear model, 
Regression trees, Scalable algorithms.

\section{Introduction}

Despite their long history, decision trees such as CART \citep{cart} and C4.5 \citep{C4.5} remain popular machine learning tools. Decision trees can be trained quickly with few parameters to tune, and the final model can be easily interpreted and visualized which is an appealing advantage in practice. As a result, these methods are widely applied in a variety of disciplines including engineering \citep{M5app5}, bioinformatics \citep{dtapp2}, agriculture \citep{dtapp3}, and business analysis \citep{dtapp1,M5app4}. In addition to their standalone use, decision trees have seen widespread adoption in ensemble methods, often as the best available ``weak learner''. Prime examples are random forests \citep{breiman2001random} and gradient boosting methods such as XGBoost \citep{xgboost} and LightGBM \citep{lightgbm}. In this work we assume that the target variable is continuous, so we focus on regression trees rather than on classification.

A limitation of classical regression trees is that their piecewise constant nature makes them ill-suited to capture continuous relationships. They require many splits in order to approximate linear functions, which is undesirable. There are two main approaches to overcome this issue. The first is to use ensembles of decision trees such as random forests or gradient boosted trees. These ensembles smooth the prediction and can therefore model a more continuous relation between predictors and response. A drawback of these ensembles is the loss of interpretability. Combining multiple regression trees no longer allows for a simple visualization of the model, or for an explainable stepwise path from predictors to prediction. The second approach to  capture continuous relationships is to use model trees. Model trees have a tree-like structure, but allow for non-constant fits in the leaf nodes of the tree. Model trees thus retain the intuitive interpretability of classical regression trees while being more flexible. Arguably the most intuitive and common model tree is the linear model tree, which allows for linear models in the leaf nodes.

To the best of our knowledge, the first linear model tree algorithm was introduced by \cite{fried}. We will abbreviate it as FRIED. In each node it uses univariate piecewise linear fits, to replace the piecewise constant fits in CART. Once a model is fit on a node, its residuals are passed on to its child nodes for further fitting. The final prediction is given by the sum of the linear models along the path. Our experiments suggest that this method may suffer from overfitting and extrapolation issues. The FRIED algorithm received much less attention than its more involved successor MARS \citep{friedman1991multivariate}.

The M5 algorithm \citep{M5} is by far the most popular linear model tree, and is still commonly used today \citep{M5app3,M5app1,M5app4,M5app2,M5app5}. It starts by fitting CART. Once this tree has been built, linear regressions are introduced at the leaf nodes. Pruning and smoothing are then applied to reduce its generalization error. One potential objection against this algorithm is that the tree structure is built completely oblivious of the fact that linear models will be used in the leaf nodes.

The GUIDE algorithm \citep{guide} fits multiple linear models to numerical predictors in each node and then applies a $\chi^2$ test comparing positive and negative residuals to decide on which predictor to split. The algorithm can also be applied to detect interactions between predictors. However, no theoretical guarantee is provided to justify the splitting procedure. There are also algorithms that use ideas from clustering in their splitting rule. For example, SECRET \citep{secret} uses the EM algorithm to fit two Gaussian clusters to the data, and locally transforms the regression problem into a classification problem based on the closeness to these clusters. Experimental results do not favor this method over GUIDE, and its computational cost is high.

The SMOTI method \citep{smoti} uses two types of nodes: regression nodes and splitting nodes. In each leaf, the final model is the multiple regression fit to its `active' variables. This is achieved by `regressing out' the variable of a regression node from both the response and the other variables. The resulting fit is quite attractive, but the algorithm has a complexity of $\mathcal{O}(n^2p^3)$, which can be prohibitive for large data sets. This makes it less suitable for an extension to random forests or gradient boosting.

Most of the above algorithms need a pruning procedure to ensure their generalization ability, which is typically time consuming. An exception is LLRT \citep{llrt} which uses stepwise regression and evaluates the models in each node via $k$-fold cross validation. To alleviate the computational cost, the algorithm uses quantiles of the predictors as the potential splitting points. It also maintains the data matrices for fitting linear models on both child nodes so that they can be updated. 
Unfortunately, the time complexity of LLRT is quite high at $\mathcal{O}(knp^3 + n p^5)$ where $k$ is the depth of the tree.

It is also interesting to compare ensembles of linear model trees with ensembles of classical decision trees. Recently, \cite{gbplt} combined piecewise linear trees with gradient boosting. For each tree, they used additive fitting as in \cite{fried} or half additive fitting where the past prediction was multiplied by a factor. To control the complexity of the tree, they set a tuning parameter for the maximal number of predictors that can be used along each path. Empirical results suggest that this procedure outperformed classical decision trees in XGBoost \citep{xgboost} and in LightGBM \citep{lightgbm} on a variety of data sets, and required fewer training iterations. This points to a potential strength of linear model trees for ensembles.

To conclude, the main issue with existing linear model trees is their high computational cost. Methods that apply multiple regression fits to leaf and/or internal nodes introduce a factor $p^2$ or $p^3$ in their time complexity. And the methods that use simple linear fits, such as FRIED, still require pruning which is also costly on large data sets. In addition, these methods can have large extrapolation errors \citep{extrapolation}. Finally, to the best of our knowledge there is no theoretical support for linear model trees in the literature. 

In response to this challenging combination of issues we propose a novel linear model tree algorithm. Its acronym PILOT stands for \textbf{PI}ecewise \textbf{L}inear \textbf{O}rganic \textbf{T}ree, where `organic' refers to the fact that no pruning is carried out. The main features of PILOT are:
  
\begin{itemize}
    \item \textbf{Speed}: It has the same low time complexity as CART without its pruning.
    \item \textbf{Regularized}: In each node, a model selection procedure is applied to the potential linear models. This requires no extra computational complexity.
    \item \textbf{Explainable}: Thanks to the simplicity of linear models in the leaf nodes, the final tree remains highly interpretable. Also a measure of feature importance can be computed.
    \item \textbf{Stable extrapolation}: Two truncation procedures are applied to avoid extreme fits on the training data and large extrapolation errors on the test data, which are common issues in linear model trees.
    \item \textbf{Theoretically supported}: PILOT has proven consistency in additive models. When the data is generated by a linear model, PILOT attains a polynomial convergence rate, in contrast with CART.
\end{itemize}
The paper is organized as follows. Section \ref{sec:description} describes the PILOT algorithm and discusses its properties, including a derivation of its time and space complexity. Section \ref{sec:theory} presents the two main theoretical results. First, the consistency for general additive models is discussed and proven. Second, when the true underlying function is indeed linear an improved rate of convergence is demonstrated. We refer to the Appendix for proofs of the theorems, propositions and lemmas. Empirical evaluations are provided in Section \ref{sec:empirical}, where PILOT is compared with several alternatives on a variety of benchmark data sets. It outperformed other tree-based methods on data sets where linear models are known to fit well, and outperformed other linear model trees on data sets where CART typically performs well. Section \ref{sec:conclusions} concludes.

\section{Methodology} \label{sec:description}
In this section we describe the workings of the PILOT learning algorithm. We begin by explaining how its tree is built and motivate the choices made, and then derive its computational cost.

We will denote the $n \times p$ design matrix as $\bX=({X}_1, \dots,{X}_n)^\top$ and the $n \times 1$ response vector as $Y=(y_1,\dots,y_n)^\top$. We consider the standard regression model $y = f(X) + \epsilon$, where $f: \mathbb R^p\rightarrow\mathbb R$ is the unknown regression function and $\epsilon$ has mean zero.

As a side remark, the linear models and trees we consider here are equivariant to adding a constant to the response, that is, the predictive models would keep the same fitted parameters except for the intercept terms. Therefore, in the presentation we will assume that the responses are centered around zero, that is, $Y_{min} = -Y_{max}$ without loss of generality.

\subsection{Main structure of PILOT} \label{sec:mainstructure}
A typical regression tree has four ingredients: a construction rule, the evaluation of the models/splits, a stopping rule, and the prediction. Most regression tree algorithms are built greedily from top to bottom. That is, they split the original space along a predictor and repeat the procedure on the subsets. This approach has some major advantages. The first is its speed. The second is that it starts by investigating the data from a global point of view, which is helpful when detecting linear relationships.

Algorithm \ref{alg:overview} presents a high-level description of PILOT. It starts by sorting each predictor and storing its ranks.
At each node, PILOT then selects a predictor and a univariate piecewise linear model via a selection procedure that we will describe in Sections \ref{sec:models} and \ref{sec:modelselect}. Then it fits the model and passes the residuals on to the child nodes for further fitting. This recursion is applied until the algorithm stops. There are three stopping triggers. In particular, the recursion in a node stops when:
\begin{itemize}
\item the number of cases in the node is below a threshold value $n_{\text{fit}}$;
\item the number of cases in one of the child nodes would be less than a threshold $n_{\text{leaf}}$;
\item none of the candidate models do substantially better than a constant prediction, as we will describe in Section \ref{sec:modelselect}.
\end{itemize}
When all nodes have reached a stopping trigger, the algorithm is done.
The final prediction is given by aggregating the predictions of the linear models from the root to the leaves, as in the example in Figure \ref{treeplot}.

\RestyleAlgo{ruled}
\SetKwComment{Comment}{/* }{ */}
\SetKwInOut{Global}{Globals}
\normalem
\begin{algorithm}[hbt!]
\caption{Sketch of the PILOT algorithm}\label{alg:overview}
\KwIn{Training set $(\bX,Y)$, $n_{\text{fit}}$: minimal number of cases to fit a model; $n_{\text{leaf}}$: minimal number of cases in a leaf;  $K_{\text{max}}$: maximal tree depth. }
\KwOut{$T$, a decision tree object.}
\SetKwFunction{FMain}{PILOT}
\SetKwProg{Fn}{function}{:}{}
   \Fn{\FMain{$\bX,Y,n_{\text{fit}},n_{\text{leaf}},K_{\text{max}}$}}{
        $I\leftarrow\{1,\dots,\mbox{length}(Y)\}$\\
        $Y_0\leftarrow Y$; $p\leftarrow$ number of predictors\\
        $B\leftarrow Y_{max}$
        
       Sort the values $X^{(j)}_i$ in each predictor variable $X^{(j)}$, and store the indices of the sorted cases in an $n \times p$ integer matrix $idx$. \\
        
       Recursively build the tree using $(\bX,Y),n_{\text{fit}},n_{\text{leaf}},K_{\text{max}},I,Y_0,p,idx,B$, and function \texttt{Model\_Selection}. Fit the selected model, update the residuals $Y$, and the index set $I$ for cases in each child node.\\ 
       \Return{T}{}}
\textbf{end}\\
\end{algorithm}
\ULforem

\begin{figure}[!ht]
\center{\includegraphics[scale=0.90]
   {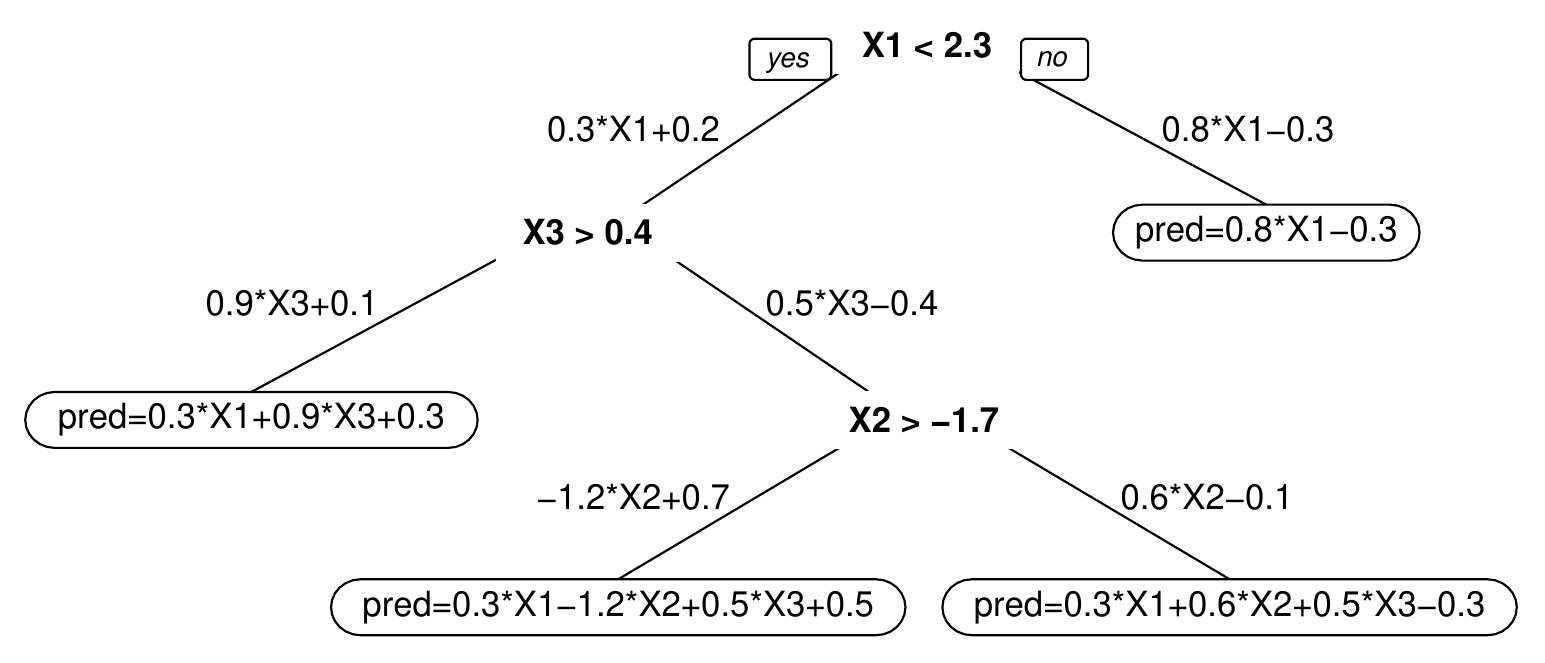}}
\caption{An example of a PILOT tree.} 
\label{treeplot}
\end{figure}

Note that PILOT retains good interpretability. The tree structure visualizes the decision process, while the simple linear models in each step reveal which predictor was used and how the prediction changes with respect to a small change in this predictor. Moreover, one can define a measure of feature importance similar to that of CART, based on the impurity gain of each node. This is because PILOT selects only one predictor in each node, which makes the gain fully dependent on it. This differs from methods such as M5, GUIDE, and LLRT that use multiple predictors in each node, making it harder to fairly distribute the gain over several predictors in order to derive an overall measure of each predictor's importance.

\subsection{Models used in the nodes}\label{sec:models}

As stated in the high-level summary in Algorithm \ref{alg:overview}, at each node PILOT selects a fit from a number of linear and piecewise linear models. In particular, PILOT considers the following regression models on each predictor:

\begin{itemize}
    \item \textsc{con}: A \textbf{CON}stant fit. We stop the recursion in a node after a \textsc{con} model is fitted.
    \item \textsc{lin}: Simple \textbf{LIN}ear regression. 
    \item \textsc{pcon}: A \textbf{P}iecewise \textbf{CON}stant fit, as in CART.
    \item \textsc{blin}: A \textbf{B}roken \textbf{LIN}ear fit: a continuous function consisting of two linear pieces.
    \item \textsc{plin}: A two-\textbf{P}iece \textbf{LIN}ear fit that need not be continuous.
\end{itemize}
To guard against unstable fitting, the \textsc{lin} and \textsc{blin} models are only considered when the number of unique values in the predictor is at least 5. Similarly, the \textsc{plin} model is only considered if both potential child nodes have at least 5 unique values of the predictor. 

These models extend the piecewise constant \textsc{pcon} fits of CART to linear fits. In particular, \textsc{plin} is a direct extension of \textsc{pcon} while \textsc{blin} can be regarded as a regularized and smoothed version of \textsc{plin}. For variables that are categorical or only have a few unique values, only \textsc{pcon} and \textsc{con} are used. Finally, if \textsc{con} is selected in a node, the recursion in that node stops.

PILOT reduces to CART (without its pruning) if only \textsc{pcon} models are selected, since the least squares constant fit in a child node equals its average response.

Note that a node will not be split when \textsc{lin} is selected. Therefore, when a \textsc{lin} fit is carried out in a node, PILOT does not increase its reported depth. This affects the reported depth of the final tree, which is a tuning parameter of the method. 

It is possible that a consecutive series of \textsc{lin} fits is made in the same node, and this did happen in our empirical studies.
In a node where this occurs, PILOT is in fact executing $L^2$ boosting  \citep{l2boosting}, which fits multiple regression using repeated simple linear regressions. It has been shown that $L^2$ boosting  is consistent for high dimensional linear regression and produces results comparable to the Lasso. It was also shown in \cite{NPBLR} that its convergence rate can be relatively fast under certain assumptions. PILOT does not increment the depth value of the node for \textsc{lin} models to avoid interrupting this boosting procedure.  

\subsection{Model selection rule}\label{sec:modelselect}

In each node, PILOT employs a selection procedure to choose between the five model types \textsc{con}, \textsc{lin}, \textsc{pcon}, \textsc{blin} and \textsc{plin}. It would be inappropriate to select the model based on the largest reduction in the residual sum of squares ($\mathrm{RSS}$), since this would always choose \textsc{plin} as that model generalizes all the others. Therefore, we need some kind of regularization to select a simpler model when the extra $\mathrm{RSS}$ gain of going all the way to \textsc{plin} is not substantial enough.

After many experiments, the following regularization scheme was adopted. 
In each node, PILOT chooses the combination of a predictor and a regression model that leads to the lowest BIC value
\begin{equation} \label{eq:BIC}
  n\log\Big(\frac{\mathrm{RSS}}{n}\Big)+v\log(n)
\end{equation}
which is a function of the residual sum of squares $\mathrm{RSS}$, the number of cases $n$ in the node, and the degrees of freedom $v$ of the model. With this selection rule PILOT mitigates overfitting locally, and applying it throughout the tree leads to a global regularization effect. 

It remains to be addressed what the value of $v$ in \eqref{eq:BIC} should be for each of the five models. The degrees of freedom are determined by aggregating the number of model parameters excluding the splits, and the degrees of freedom of a split. The model types \textsc{con}, \textsc{lin}, \textsc{pcon}, \textsc{blin} and \textsc{plin} contain 1, 2, 2, 3, and 4 coefficients apart from any split points. Based on empirical and theoretical arguments, \cite{discont} found that a discontinuity point deserves 3 degrees of freedom. We follow this recommendation for \textsc{pcon} and \textsc{plin}, which each receive 3 additional degrees of freedom. Also \textsc{blin} contains a split point, but here the fitted function is continuous. To reflect this intermediate complexity we add 2 degrees of freedom to \textsc{blin}. In conclusion, we end up with $v$ = 1, 2, 5, 5, and 7 for the model types \textsc{con}, \textsc{lin}, \textsc{pcon}, \textsc{blin} and \textsc{plin}.

The BIC in~\eqref{eq:BIC} is one of several model selection criteria that we could have used used.
We also tried other measures, such as the Akaike information criterion (AIC) and the adjusted AIC. It turned out that the AIC tended to choose \textsc{plin} too often, which reduced the regularization effect. The adjusted AIC required a pruning procedure to perform comparably to the BIC criterion. As pruning comes at a substantial additional computational cost, we decided against the adjusted AIC.

Alternatively, one could follow \cite{l2boosting} to compute the degrees of freedom for the aggregated model in each node, and then compute the adjusted AIC to decide when to stop. But in this approach hat matrices have to be maintained for each path, which requires more computation time and memory space. Moreover, as the number of cases in the nodes changes, the evaluations of the hat matrices become more complicated. For these reasons, we preferred to stay with the selection rule~\eqref{eq:BIC} in each node.

Algorithm~\ref{alg:modelsel} shows the pseudocode of our model selection procedure in a given node. We iterate through each predictor separately. If the predictor is numerical we find the univariate model with lowest BIC among the 5 candidate models. For \textsc{con} and \textsc{lin} this evaluation is immediate. For the other three models, we need to consider all possible split points. We consider these in increasing order (recall that the ordering of each predictor was already obtained at the start of the algorithm), which allows the Gram and moment matrices for the linear models to be maintained and updated efficiently in each step. On the other hand, if the predictor is categorical, we follow the approach of CART as in Section 9.2.4 of \cite{esl}. That is, we first compute the mean $m_c$ of the response of each level $c$ and then sort the cases according to $m_c$. We then fit the model \textsc{pcon} in the usual way.

Once the best model type has been determined for each predictor separately, Algorithm~\ref{alg:modelsel} returns the combination of predictor and model with the lowest BIC criterion.

\RestyleAlgo{ruled}
\SetKwInOut{Global}{Globals}
\normalem
\begin{algorithm}[hbt!]
\caption{Model selection and split finding in a node of the training set}\label{alg:modelsel}
\Global {$(\bX,Y)$: predictors and residual vector in the current node; $p$: number of predictors; $idx$: indices matrix of the $p$ presorted predictors; $n_{\text{leaf}}$: minimal number of cases in the leaf.}
\KwIn{$I$, indices of cases in the current node.}
\KwOut{$best\_pivot$; $best\_model$; $best\_pred$; $lm\_l$, $lm\_r$, linear model for the left and right child; $range$, range of the selected predictors in the current node.}
\SetKwFunction{FMain}{Model\_Selection}
\SetKwProg{Fn}{function}{:}{}
\SetKwProg{Init}{init:}{}{}
\SetKwProg{Return}{return:}{}{}
    \Fn{\FMain{$I$}}{
        \Init{best\_pivot, best\_model, best\_bic, best\_pred, lm\_l, lm\_r, range}{}

        \For{each predictor $X^{(j)}$ with $j=1,\ldots,p$}
        {Select the cases sorted according to $X^{(j)}$ using $I$ and $idx$.

        Fit \textsc{con} and initialize the parameters of the best model using its result.\\
        \uIf{$X^{(j)}$ is a numerical predictor}
            {Compute $(D^{\textsc{lin}})^\top D^{\textsc{lin}}$ and $(D^{\textsc{lin}})^\top Y$, where $D^{\textsc{lin}}$ is the     design matrix of the linear model. It is $|I| \times 1$ for \textsc{pcon}, and  $|I| \times 2$ otherwise. \\
            
            Fit \textsc{lin} using the above matrices. Store the parameters of the best model. 
            
            For \textsc{pcon} and \textsc{plin} initialize the Gram matrices $G^{\textsc{plin}}_{l,r}$ and moment matrices $M^{\textsc{plin}}_{l,r}$ of the simple linear models in the left and right child such that $G^{\textsc{plin}}_{l}=M^{\textsc{plin}}_l=\boldsymbol 0$, $G^{\textsc{plin}}_{r}=(D^{\textsc{lin}})^\top D^{\textsc{lin}}$, $M^{\textsc{plin}}_r=(D^{\textsc{lin}})^\top Y$.
             
            For \textsc{blin} initialize the Gram matrix $G^{\textsc{blin}} = (D^{\textsc{blin}})^\top D^{\textsc{blin}}$ and the moment matrix $M^{\textsc{blin}}=(D^{\textsc{blin}})^\top Y$ where $D^{\textsc{blin}}$ is the design matrix of \textsc{blin} with knot equal to $\min X^{(j)}$. 
           
            \For{each possible splitting pivot}{
            {Update $G^{\textsc{blin}},M 
  ^{\textsc{blin}}$, $G^{\textsc{plin}}_{l,r},M^{\textsc{plin}}_{l,r}$. 
               
           {If $n_{\text{leaf}}$ is satisfied, fit \textsc{plin}, \textsc{pcon} and \textsc{blin} using the Gram and moment matrices, yielding the left and right models $lm\_l$ and $lm\_r$.}\\
           Store the parameters of the best model.}
           }
           }
        \uElse 
            {Compute the mean $m_c$ of the response for each categorical level $c$.
            
            Re-sort the cases according to their $m_c$.
            
            Greedily search the best $value$ such that \textsc{pcon} with partition \{$c|m_c\leqslant value$\} and \{$c|m_c>value$\} has the lowest BIC.
            
            Store the parameters of the best model. Pivot is the set \{$c|m_c \leqslant  value$\}.} 
           }     
    \Return{best\_pivot, best\_model, best\_bic, best\_pred, lm\_l, lm\_r, range}{} }
\textbf{end}\\
\end{algorithm}
\ULforem

\subsection{Truncation of predictions} \label{sec:trunc}
PILOT relies on two truncation procedures to avoid unstable predictions on the training data and the test data. 

\begin{figure}[!ht]
\centering
\includegraphics[width=0.49\textwidth]
   {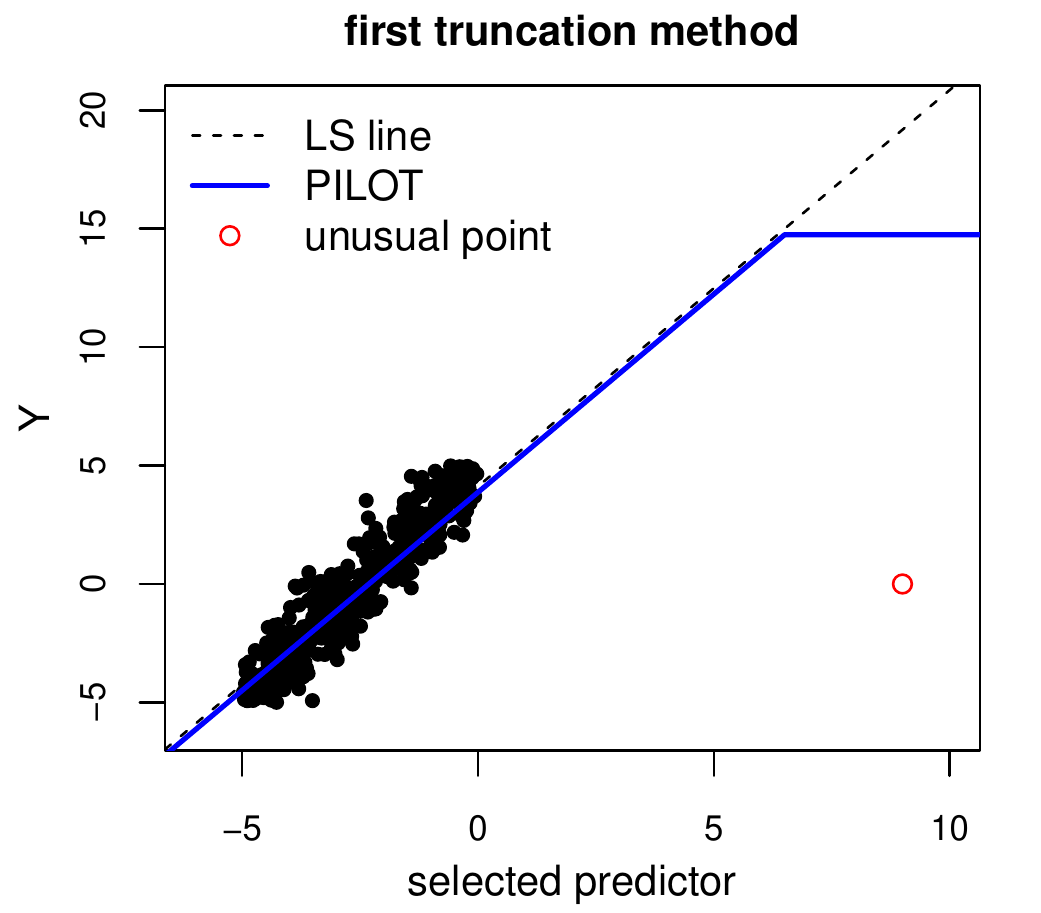}
\includegraphics[width=0.49\textwidth]
   {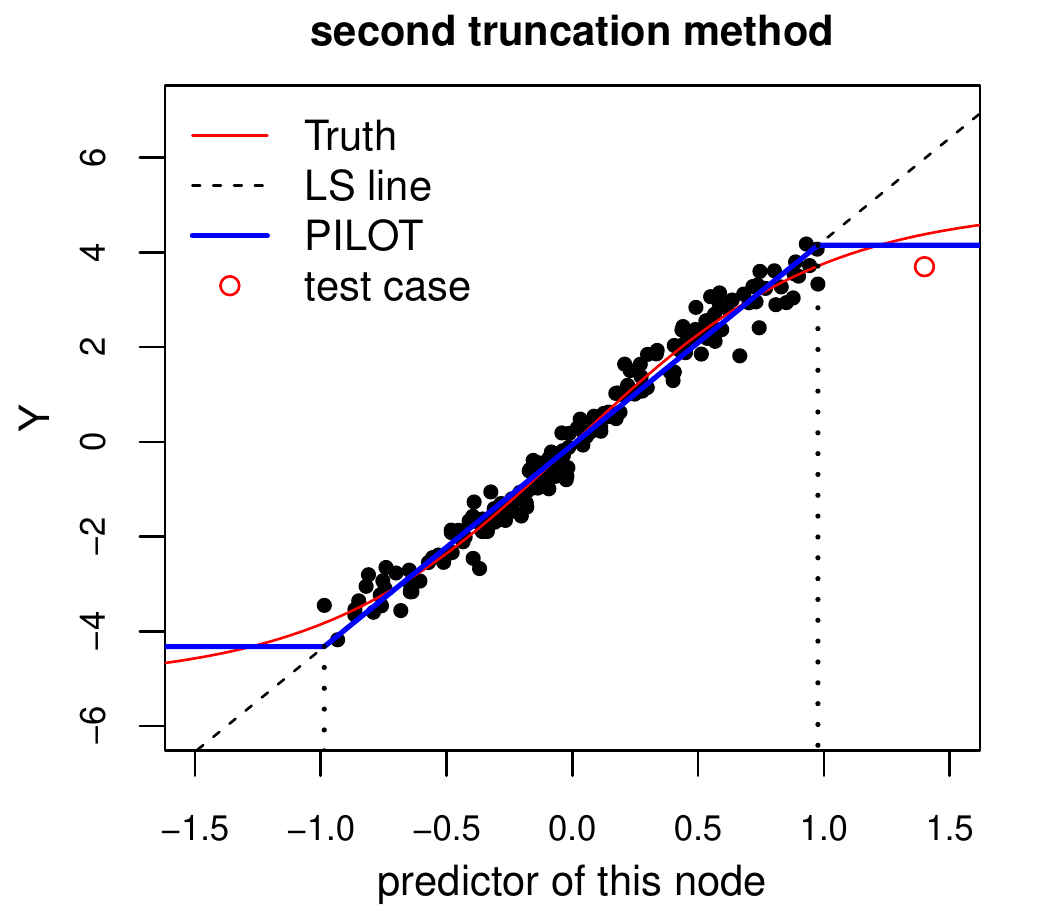}
\caption{Left: An example of the first truncation method in one node of the tree. Right: An example of the second truncation method.}
\label{truncations}
\end{figure}

The first truncation procedure is motivated as follows. The left panel of Figure \ref{truncations} shows the data in a node of PILOT. The selected predictor is on the horizontal axis. The plot illustrates a situation where most of the data points in a node are concentrated in a part of the predictor range, causing the linear model to put a lot of weight on these points and little weight on the other(s). This can result in extreme predictions, for instance for the red point on the right. Although such unstable predictions might be corrected by deeper nodes, this could induce unnecessarily complicated models with extra splits. Moreover, if this effect occurs at a leaf node, the prediction cannot be corrected. To avoid this issue we will truncate the prediction function. Note that if we add a constant to the response, PILOT will yield the same estimated parameters except for the intercept terms that change the way they should. Therefore, we can assume without loss of generality that $Y_{max}=-Y_{min}$\,. Calling this quantity $B$, the original responses have the range $[-B,B]$. We then clip the prediction so it belongs to $[-3B,3B]$. Then we compute the response $Y$ minus the truncated prediction to get the new residuals, and proceed with building the tree for the next depth. The whole process is summarized in Algorithm \ref{alg:buildTree}. Truncating at $\pm 3B$ may seem rather loose, and one could truncate at a smaller multiple of $B$ if so desired. But we found empirically that truncating at $\pm 3B$ works fine for our purpose.

\normalem
\begin{algorithm}[hbt!]
\caption{Truncation during tree building on the training data}\label{alg:buildTree}
\Global{$p$; $idx$; $n_{\text{fit}}$: minimal number of cases required to fit a model; $n_{\text{leaf}}$: minimal number of cases in a leaf; $K_{\text{max}}$: maximal depth; $Y_0$: original response; $B$: maximum of centered response.}
\KwIn{$(\bX,Y)$: Training set; $K$: current depth; $I$: indices of cases. }
\KwOut{A tree object $T$ with attributes 
$m$: the model; $l,r$: the left and the right child node; 
$lm_l,lm_r$: model coefficients in the left and right child nodes.}
\SetKwFunction{FMain}{Build\_Tree}
\SetKwProg{Fn}{function}{:}{}
    \Fn{\FMain{$(\bX,Y),K,I$}}{
        $T\leftarrow$ \text{a None node}\\
         \uIf{$K< K_{max}\text{ and length($I$)}\geqslant n_{\text{fit}}$}
        {
        Do \texttt{Model\_Selection}($I$) and record the parameters for the best model\\
        $K\leftarrow K+1$ if the $best\_model$ is not \textsc{lin} or \textsc{con}\\
        $T\leftarrow$ A tree node with the recorded parameters\\
        \uIf{best\_model is \textsc{con}}{\Return{$T$}{}}
        \uElseIf{best\_model is \textsc{lin}}
        {
        $raw\_pred[I]\;\leftarrow\;$ the prediction of $lm_l$ on $\bX[I]$\\
        Truncate the actual prediction $Y_0[I]-Y[I]+raw\_pred[I]$ at $\pm3B$. Update the residuals $Y[I]$ with this new prediction.\\
        $T.l\leftarrow \texttt{Build\_Tree}((\bX,Y),K,I)$}
        \Else{
        $I_l$, $I_r\leftarrow$ the indices of cases in the left and right child\\
        $raw\_pred[I_l]\;\leftarrow\;$ the prediction of $lm_l$ on $\bX[I_l]$\\
        $raw\_pred[I_r]\;\leftarrow\;$ the prediction of $lm_r$ on $\bX[I_r]$\\    
        Truncate the actual predictions $Y_0[I_l] - Y[I_l] + raw\_pred[I_l]$ and 
        $Y_0[I_r] - Y[I_r] + raw\_pred[I_r]$ at $\pm 3B$, and update the residuals $Y[I_l]$ and $Y[I_r]$ in the child nodes. \\
        $T.l\leftarrow \texttt{Build\_Tree}((\bX,Y),K,I_l)$\\
        $T.r\leftarrow \texttt{Build\_Tree}((\bX,Y),K,I_r)$
        }
        }
        \Else{\Return{$T$}{}}
        \textbf{return} $T$}
\end{algorithm}
\ULforem  

The first truncation method is not only applied during training, but also to the prediction on new data. The range of the response of the training data is stored, and when a prediction for new data would be outside the proposed range $[-3B,3B]$ it is truncated in the same way. This procedure works as a basic safeguard for the prediction. However, this is not enough because a new issue can arise: it may happen that values of a predictor in the test data fall outside the range of the same predictor in the training data. This is illustrated in the right panel of Figure \ref{truncations}. The vertical dashed lines indicate the range of the predictor on the training data. However, the predictor value of the test case (the red point) lies quite far from that range. Predicting its response by the straight line that was fitted during training would be an extrapolation, which could be unrealistic. This is a known problem of linear model trees. For instance, in our experiments in Section \ref{sec:empirical} the predictions of FRIED and M5 exploded on some data sets, which turned out to be due to this problem.

Therefore we add a second truncation procedure, which is only applied to new data (because in the training data, all predictor values are inside their range). Therefore, we propose to truncate the prediction value as done by  \cite{extrapolation}. More precisely, during training we record the range $[x_{\mbox{\tiny min}}, x_{\mbox{\tiny max}}]$ of the predictor selected in this node, and store the range of the corresponding predictions $[\hat f(x_{\mbox{\tiny min}}),\hat f(x_{\mbox{\tiny max}})]$. When predicting on new data, we truncate the predictions so they stay in the training range. More precisely, we replace the original predictions $\hat f(x_{\mbox{\tiny test}})$ on this node by $\max(\min(\hat f(x_{\mbox{\tiny test}}),\hat f(x_{\mbox{\tiny max}})),\hat f(x_{\mbox{\tiny min}}))$. 

The two truncation methods complement each other. For instance, the first approach would not suffice in the right panel of Figure \ref{truncations} since the linear prediction at the new case would still lie in $[-3B,3B]$. Also, the second approach would not work in the left panel of Figure \ref{truncations} as the unusual case is included in the range of the training data. Our empirical studies indicate that combining both truncation approaches helps eliminate extreme predictions and therefore improve the stability of our model. Moreover, this did not detract from the overall predictive power. Algorithm \ref{alg:predict} describes how both truncation methods are applied when PILOT makes predictions on new data.

\normalem
\begin{algorithm}[hbt!]
\caption{Truncation during prediction on new data}\label{alg:predict}
\KwIn{$T$: a decision tree object; $x$: a new observation; $B$: maximum value of the centered response in the training data. }
\KwOut{$pred$, the prediction in $x$.}
\SetKwFunction{FMain}{Predict}
\SetKwProg{Fn}{function}{:}{}
    \Fn{\FMain{$T,X$}}{
        $pred\leftarrow0$
        
        \While{T is not None}
        {{\uIf {$x$ is in the left child of $T$}
          {$X_{min}, X_{max}\leftarrow$ The lower and upper bound of $T.range$ on the selected predictor $j$ in the training data.\\
          $X^*\leftarrow\min(\max(X^{(j)},X_{min}),X_{max})$
          
          $pred\_add$ $\leftarrow$ The prediction of $T.lm\_l$ on $X^*$

          $pred$ $\leftarrow$ $\max(\min(pred + pred\_add,3B)-3B)$
          
          $T\leftarrow T.l$} 
         \Else{
         $X_{min}, X_{max}\leftarrow$ The lower and upper bound of $T.range$ on the selected predictor $j$ in the training data.\\
         $X^*\leftarrow\max(\min(X^{(j)},X_{min}),X_{max})$

         $pred\_add$ $\leftarrow$ The prediction of $T.lm\_r$ on $X^*$
          
         $pred$ $\leftarrow$ $\max(\min(pred + pred\_add,3B),-3B)$
          
         $T\leftarrow T.r$}}}
        \Return{$pred$}{}}
\textbf{end}\\
\end{algorithm}
\ULforem

\subsection{Stopping rules versus pruning}

It is well known that decision trees have a tendency to overfit the data if the tree is allowed to become very deep. A first step toward avoiding this is to require a minimal number of cases in a node before it can be split, and a minimal number of cases in each child node. In addition, PILOT also stops splitting a node if the BIC model selection criterion selects \textsc{con}. 

Several decision trees, including CART, take a different approach. They let the tree grow, and afterward prune it. This indeed helps to achieve a lower generalization error. However, pruning can be very time consuming. One often needs a cross-validation procedure to select pruning parameters. It would be possible for PILOT to also incorporate cost-complexity pruning as in CART. (If a series of \textsc{lin} fits occurs in the same node, PILOT would either keep all of them or remove them together.) However, in our empirical study we found that adding this pruning to PILOT did not outperform the existing \textsc{con} stopping rule on a variety of datasets. Also, we will see in the next section that the generalization error vanishes asymptotically. Moreover, not pruning is much more efficient computationally. For these reasons, PILOT does not employ pruning. The computational gain of this choice makes it more feasible to use PILOT in ensemble methods.

\subsection{Time and space complexity}

Finally, we demonstrate that the time complexity and space complexity of PILOT are the same as for CART without its pruning. The main point of the proof is that for a single predictor, each of the five models can be evaluated in a single pass through the predictor values. Evaluating the next split point requires only a rank-one update of the Gram and moment matrices, which can be done in constant time. Naturally, the implicit proportionality factors in front of the $\mathcal{O}(.)$ complexities are higher for PILOT than for CART, but the algorithm is quite fast in practice.

\begin{proposition}\label{complexity}
PILOT has the same time and space complexities as CART without its pruning.
\end{proposition}

\begin{proof}
 For both PILOT and CART we assume that the $p$ predictors have been presorted, which only takes $\mathcal{O}(np\log(n))$ time once. 
 
 We first check the complexity of the model selection in Algorithm \ref{alg:modelsel}. It is known that the time complexity of CART for split finding is $\mathcal{O}(np)$. For the \textsc{pcon} model, PILOT uses the same algorithm as CART. For \textsc{con} we only need to compute one average in $\mathcal{O}(n)$ time. For \textsc{lin}, the evaluation of $(D^{\textsc{lin}})^\top D^{\textsc{lin}}$ and $(D^{\textsc{lin}})^\top Y$ also takes $\mathcal{O}(n)$. In the evaluation of the \textsc{plin} model, the Gram and moment matrices always satisfy $G^{\textsc{plin}}_l+G^{\textsc{plin}}_r=G^{\textsc{lin}}=(D^{\textsc{lin}})^\top D^{\textsc{lin}}$ and $M^{\textsc{plin}}_l+M^{\textsc{plin}}_r=(D^{\textsc{lin}})^\top Y$ by definition. These matrices have at most 4 entries, and can be incrementally updated in $\mathcal{O}(1)$ time as we evaluate the next split point. For \textsc{blin} the reasoning is analogous. 
In each model, inverting the Gram matrix only takes $\mathcal{O}(1)$ because its size is fixed. Therefore, we can evaluate all options on one presorted predictor in one pass through the data, which remains $\mathcal{O}(n)$, so for all predictors this takes $\mathcal{O}(np)$ time, the same as for CART. For the space complexity, the Gram and moment matrices only require $\mathcal{O}(1)$ of storage. Since CART also has to store the average response in each child node, which takes $\mathcal{O}(1)$ storage as well, the space complexity of both methods is the same.

For the tree building in Algorithm \ref{alg:buildTree}, the time complexity of computing the indices $I_{l,r}$ and the residuals is $\mathcal{O}(n)$ at each depth $1 \leqslant  k \leqslant  K_{max}$. CART also computes the indices $I_{l,r}$ which requires the same complexity $\mathcal{O}(n)$ for each step. Therefore, the overall time for both methods is $\mathcal{O}(K_{max}n)$. Since the PILOT tree only has $\mathcal{O}(1)$ more attributes in each node than CART, the space complexity of PILOT remains the same as that of CART.

For the initialization and prediction parts, the results are straightforward.
\end{proof}

\section{Theoretical results}\label{sec:theory}
\subsection{Universal consistency}
In this section we prove the consistency of PILOT. We follow \cite{UCDT} and assume the underlying function $f\in\mathcal{F}\subset L^2([0,1]^p)$ admits an additive form
\begin{equation} \label{eq:additive}
  f(X):=f_1(X^{(1)})+\dots+f_p(X^{(p)})
\end{equation}
where $f_i$ has bounded variation and $X^{(j)}$ is the $j$-th predictor. We define the total variation norm $||f||_{TV}$ of $f\in\mathcal{F}$ as the infimum of $\sum_{i=1}^p||f_i||_{TV}$ over all possible representations of $f$, and assume that the representation in \eqref{eq:additive} attains this infimum.

For all $f, g \in L^2([0,1]^p)$ and a dataset $X_1, \ldots, X_n$ of size $n$, we define the empirical norm and the empirical inner product as 
\begin{equation}
  ||f||_n^2\,:=\frac{1}{n}\sum_{i=1}^{n}|f( X_i)|^2 \quad \text{and} \quad \langle f,g \rangle_n:=\frac{1}{n}\sum_{i=1}^{n}f( X_i)g( X_i).
\end{equation}
For the response vector $Y = (Y_1,\ldots, Y_n)$ we denote
\begin{equation}
  ||Y-f||_n^2\,:=\frac{1}{n}\sum_{i=1}^{n}(Y_i - f( X_i))^2 \quad \text{and} \quad \langle Y,f \rangle_n:=\frac{1}{n}\sum_{i=1}^{n}Y_if( X_i).
\end{equation}
The general $L^2$ norm on vectors and functions will be denoted as $||\cdot||$ without subscript.

To indicate the norm and the inner product of a function on a node $T$ with t observations we replace the subscript $n$ by $t$. We denote by $\mathcal{T}_k$ the set of tree nodes at depth $k$, plus the leaf nodes of depth lower than $k$. In particular, $\mathcal{T}_K$ contains all the leaf nodes of a $K$-depth tree.

Now we can state the main theorem:
\begin{theorem}\label{theorem:consistency}
Let $f\in \mathcal{F}$ with\ $||f||_{TV}\, \leqslant A$ and denote by $\hat f(\mathcal{T}_K)$ the prediction of a \mbox{$K$-depth} PILOT tree. Suppose $X\sim P$ on $[0,1]^{p_n}$, the responses  are a.s.\ bounded in $[-B,B]$, and the depth $K_n$ and the number of predictors $p_n$ satisfy $K_n\rightarrow\infty$ and $2^{K_n}p_n\log(np_n)/n\rightarrow0$. Then PILOT is consistent, that is
\begin{equation}
\lim_{n\rightarrow\infty} \maE[||f-\hat f(\mathcal T_{K_n})||^2]=0.
\end{equation}
\end{theorem}

\begin{remark}
Note that the conditions on the depth and the dimension can easily be satisfied if we let $K_n=\log_2(n)/r$ for some $r>1$ and $p_n=n^{s}$ such that $0<s<1-1/r$. The resulting convergence rate is $\mathcal{O}(1/\log(n))$. This is the same rate as obtained by \cite{UCDT} for CART under very similar assumptions.
\end{remark}

The key of the proof of Theorem~\ref{theorem:consistency} is to establish a recursive formula for the training error $R_k:=||Y-\hat f(\mathcal{T}_k)||^2_n-||Y-f||^2_n$ for each depth $k$. Then we can leverage results from empirical process theory \citep{distribution} to derive an oracle inequality for PILOT, and finally prove consistency. 

For the recursive formula we first note that $R_k=\sum_{T\in \mathcal{T}_K}w(T)R_k(T)$ where $w(T):=t/n$ is the weight of node $T$ and $R_k(T):=||Y-\hat f(\mathcal{T}_k)||^2_t-||Y-f||^2_t$. Then we immediately have $R_{k+1}=R_k-\sum_{T\in \mathcal{T}_k}w(T)\Delta^{k+1}(T)$ where
$$\Delta^{k+1}(T):=||Y-\hat{f}(\mathcal{T}_k)||^2_t\;-\;t_l||Y-\hat{f}(\mathcal{T}_{k+1})||^2_{t_l}/t\;-\;t_r||Y-\hat{f}(\mathcal{T}_{k+1})||^2_{t_r}/t$$
is the impurity gain of the model on $T$. Here $t_l$ and $t_r$ denote the number of cases in the left and right child node. If no split occurs, one of these numbers is set to zero. For PILOT we need to remove \textsc{con} nodes that do not produce an impurity gain from the recursive formula. To do this, we define  $C_k^+=\{T|T=\textsc{con}, T\in \mathcal T_{k-1}, R_k(T)>0\}$, i.e., the set of `bad' nodes on which \textsc{con} is fitted, and similarly $C_k^-$ for those with $R_k(T)  \leqslant  0$. Then we can define $R_{C_k^+}:=\sum_{T\in C_k^+}w(T)R_k(T)$, the positive training error corresponding to the cases in these \textsc{con} nodes. We can then consider $\widetilde{R}_k:=R_k-R_{C_k^+}$ and show that asymptotically both $R_k$ and $R_{C_k^+}$ become small. 

Now the problem is reduced to relating the impurity gain of the selected model to the training error before this step. Recently, \cite{UCDT} introduced a novel approach for this kind of estimation in CART. However, his proof used the fact that in each step the prediction is piecewise constant, which is not the case here. Therefore, we derived a generalized result for $\textsc{pcon}$ in PILOT (Lemma \ref{l3} in Appendix \ref{app:B.1}). For controlling the gain of the other models, we can use the following proposition:

\begin{proposition}\label{prop:bic}
Let $\Delta_1$, $\Delta_2$ and $v_1$, $v_2$ be the impurity gains and degrees of freedom of two regression models on some node $T$ with $t$ cases. Let $R_0$ be the initial residual sum of squares in $T$. We have that
\begin{itemize}
\item If model 1 does better than \textsc{con}, i.e.\ $BIC_{con}> BIC_1$, we have
$t\Delta_1/R_0> C(v_1,t)>0$ for some positive function $C$ depending on $v_1$ and $t$. 
\item If $BIC_{con} > BIC_1$ and $BIC_2 > BIC_1$ with $v_2 \geqslant v_1$ we have
$\frac{\Delta_1}{\Delta_2} \geqslant \frac{v_1-1}{v_2-1}.$
\end{itemize}
Moreover, if \textsc{con} is chosen at a node, it will also be chosen in subsequent models.
\end{proposition}

We thus know that the gain of the selected model is always comparable to that of \textsc{pcon}, up to a constant factor. 

The proposition also justifies the \textsc{con} stopping rule from a model selection point of view: when \textsc{con} is selected for a node, all the subsequent models for this node will still be \textsc{con}. A \textsc{con} model is selected when the gain of other regression models do not make up for the complexity they introduce. Since \textsc{con} only regresses out the mean, we can stop the algorithm within the node once the first \textsc{con} fit is encountered. 

We can show the following recursive formula:
\begin{theorem}\label{theorem:recursion1}
Under the same assumptions as Theorem \ref{theorem:consistency} we have for any $K\geqslant1$ and any $f\in \mathcal{F}$ that
\begin{equation}
||Y-\hat f(\mathcal{T}_K)||^2_n\, \leqslant ||Y-f||^2_n \, + \frac{(||f||_{TV}+6B)^2}{6(K+3)} + R_{C_K^+}\;.
\end{equation}
Moreover, if we let $K=\log_2(n)/r$ with $r>1$ we have $R_{C_K^+}\sim\mathcal{O}(\sqrt{\log (n)/n^{(r-1)/r}})$.
\end{theorem}
The BIC criterion ensures that the training errors in all \textsc{con} nodes vanish as $n\rightarrow \infty$. This allows us to prove the main theorem by using the preceding results and Theorems 11.4 and 9.4 and Lemma 13.1 of \cite{distribution}. All intermediate lemmas and proofs can be found in the Appendix \ref{app:B}.

\subsection{Convergence rates on linear models}
The convergence rate resulting from Theorem \ref{theorem:consistency} is the same as that of CART derived by \cite{UCDT}. This is because we make no specific assumptions on the properties of the true underlying additive function. Therefore this rate of convergence holds for a wide variety of functions, including poorly behaved ones. In practice however, the convergence could be much faster if the function is somehow well behaved. In particular, given that PILOT incorporates linear models in its nodes, it is likely that it will perform better when the true underlying function is linear, which is a special case of an additive model. We will show that PILOT indeed has an improved rate of convergence in that setting. This result does not hold for CART, and to the best of our knowledge was not proved for any other linear model tree algorithm.

In order to prove the convergence rate of PILOT in linear models, we apply a recent result of \cite{NPBLR}. For the convergence rate of the $L^2$ boosting algorithm, \cite{NPBLR} showed in their Theorem 2.1 that the mean squared error (MSE) decays exponentially with respect to the depth $K$, for a fixed design matrix. The proof is based on the connection between $L^2$ boosting and quadratic programming (QP) on the squared loss $||Y-\boldsymbol X\beta||^2_n$. In fact, the $L^\infty$ norm of the gradient of the squared loss is equal to the largest impurity gain among all simple linear models, assuming each predictor has been centered and standardized. Therefore, the gain of the $k$-th step depends linearly on the mean squared error in step $k-1$, which results in a faster convergence rate. By the nature of quadratic programming, the rate depends on the smallest eigenvalue $\lambda_{min}$ of $\boldsymbol{X}^\top\boldsymbol{X}$.

In our framework we can consider a local QP problem in each node to estimate the gain of \textsc{lin}, which gives a lower bound for the improvement of the selected model (except when \textsc{con} is selected). To ensure that we have a constant ratio for the error decay, we have to make some assumptions on the smallest eigenvalue of the local correlation matrix of the predictors. This eigenvalue can be regarded as a measure of multicollinearity between the predictors, which is a key factor that affects the convergence rate of least squares regression. 

To be precise, we require that the smallest eigenvalue of the theoretical correlation matrix (the $jl$-th entry of which is $\maE[(X-\bar X)(X-\bar X)^\top]_{jl}/(\sigma_j\sigma_l)$ where $\sigma_j$ is the standard deviation of the $j$-th predictor) is lower bounded by some constant $\lambda_0$ in any cubic region that is sufficiently small. Then we apply two concentration inequalities to show that the smallest eigenvalue of a covariance matrix of data in such a cube is larger than $\lambda_0$ with high probability. Finally, we can show that the expectation of the error decays exponentially,  which leads to a fast convergence rate.

\noindent Our conditions are:
\begin{itemize}
\item \textbf{Condition 1:} The PILOT algorithm stops splitting a node whenever
\begin{itemize}
\vspace{-0.2cm}
\item the number of cases in the node is less than $n_{\mbox{\tiny min}}=n^\alpha$, $0<\alpha<1$. 
\item the variance of some predictor is less than $2\sigma_0^2$ where $0<\sigma_0<1$. 
\item the volume of the node is less than a threshold $\eta$, $0<\eta<1$.
\end{itemize}
\item \textbf{Condition 2:} We assume that $X\sim P$ on $[0,1]^p$ and the error $\epsilon$ has a finite fourth moment. Moreover, for any cube $C$ with volume smaller than $\eta$ we assume that\linebreak $\lambda_{min}(Cor(X|X\in C)) \geqslant 2\lambda_0>0$, where $Cor(X)$ is the correlation matrix.
\end{itemize}

The first condition is natural for tree-based methods. For the second one, an immediate example would be a setting when all the predictors are independent. Another example is when they follow a multivariate Gaussian distribution, which becomes a truncated Gaussian when restricted to $[0,1]^p$ or to some other cube $C$. It can be shown that the Gaussian distribution has relatively small correlation in cubes $C$ of small volume; see, e.g., \cite{truncated_gaussian} for the two-dimensional case. This is not surprising since on small cubes the density tends to a constant. As a result, the smallest eigenvalue of the correlation matrix on such cubes will be bounded from below.

Under these conditions we obtain the following result, whose proof can be found in  Appendix \ref{app:C}.
\begin{theorem}\label{theorem:recursion2}
Assume the data are generated by the linear model $Y\sim\boldsymbol X\beta + \epsilon$ with $n$ cases. Under conditions 1 and 2, the difference between the training loss $L^k_n$ of PILOT at depth $k$ and the training loss $L^*_n$ of least squares linear regression satisfies
\begin{equation}
\maE[L^k_n-L^*_n] \leqslant \gamma^k\sqrt{\maE[(L^0_n-L^*_n)^2]}+\mathcal{O}(N_{\text{leaves}}\log(n)/n)
\end{equation}
where 
$$\gamma:=1-\frac{\lambda_0}{4p}\;.$$
\end{theorem}

Combining this result with the techniques in the proof of Theorem \ref{theorem:consistency}, we can show polynomial rate convergence of our method on linear functions.

\begin{corollary}[Fast convergence on linear models]\label{cor:rate}
Assume the conditions of Theorem \ref{theorem:recursion2} hold and that $|Y|$ is a.s. bounded by some constant $B$. Let $K_n= log_\gamma(n)$. Then we have for any $0<\alpha<1$ and corresponding $n_{\mbox{\tiny min}} = n^{\alpha}$ that
\begin{equation}
  \maE[||\hat f(\mathcal{T}_K) - \boldsymbol X\beta||^2] \leqslant \mathcal{O}\Big(\frac{\log(n)}{n^{\alpha}}\Big)\;.
\end{equation}
\end{corollary}

The choice of the tuning parameter $\alpha$ thus determines the convergence rate. When we use a low $\alpha$ we get smaller nodes hence more nodes, yielding looser bounds in the oracle inequality and for the errors in \textsc{con} nodes, leading to a slower convergence. This is natural, as we have fewer cases for the estimation in each node. If we choose $\alpha$ relatively high we obtain a faster rate, which is intuitive since for $\alpha \rightarrow 1$ we would have only a single node. 

\cite{obtree} prove a polynomial convergence rate for oblique decision trees with a more general function class, and it may be possible to adapt this work to PILOT. However, we note that their result makes the strong assumption that the training cases should be approximately evenly distributed over the leaf nodes. Our conditions here focus on the rules of the algorithm and the distribution of the data itself.

\section{Empirical evaluation}\label{sec:empirical}
In this section we evaluate the proposed PILOT algorithm empirically and compare its results with those of some popular competitors. We start by describing the data and methods under comparison. 

\subsection{Data sets and methods}
We analyzed 20 benchmark data sets, with the number of cases ranging from 71 to 21263 and the number of predictors ranging from 6 to 4088. From the UCI repository \citep{Dua:2019} we used the data sets \texttt{Abalone}, \texttt{Airfoil}, \texttt{Bike}, \texttt{Communities}, \texttt{Concrete}, \texttt{Diabetes}, \texttt{Electricity},   \texttt{Energy}, \texttt{Residential}, \texttt{Skills}, 
\texttt{Superconductor}, \texttt{Temperature} and \texttt{Wine}. 
From Kaggle we obtained \texttt{Graduate Admission}, \texttt{Bodyfat}, \texttt{Boston Housing} and \texttt{Walmart}. The \texttt{California Housing} data came from the StatLib repository (\url{http://lib.stat.cmu.edu/datasets/}). The \texttt{Ozone} data is in the \texttt{R}-package \texttt{hdi} \citep{hdi}, and \texttt{Riboflavin} came from the \texttt{R}-package \texttt{missMDA} \citep{missmda}. Table \ref{tab:datasets} gives an alphabetical list of the data sets with their sizes. 

\begin{table}[!ht]
\centering
\small
\caption{The data sets used in the empirical study.}
\label{tab:datasets}
\begin{tabular}{lccc} \hline
Source & Dataset& $n$ & $p$\\ 
\hline
UCI repository & Abalone &4177&8\\
Kaggle & Graduate Admission & 400&9\\
UCI repository & Airfoil & 1503&6\\
UCI repository & Bike &17389&16\\
Kaggle & Bodyfat&252&14\\
Kaggle &Boston Housing & 506&13\\
StatLib repository & California Housing & 20640&8\\
UCI repository & Communities&1994&128\\
UCI repository & Concrete &1030&9\\
UCI repository & Diabetes & 442&10\\
UCI repository & Electricity&10000&14\\
UCI repository & Energy&768&8\\
\texttt{hdi} & Ozone & 112 &11\\
UCI repository & Residential& 372&105\\
\texttt{missMDA} & Riboflavin & 71&4088\\
UCI repository & Skills&3395&20\\
UCI repository & Superconductor&21263&81\\
UCI repository & Temperature& 7750&25\\
Kaggle & Walmart &6435&8\\
UCI repository & Wine & 4898&12 \\
\hline
\end{tabular}
\end{table}

In the comparative study we ran the following tree-based methods: the proposed \mbox{PILOT} algorithm we implemented in Python with NUMBA acceleration in the split function, FRIED \citep{fried} implemented by us, M5 \citep{M5} from the \texttt{R} package Rweka \citep{hornik2009open}, and CART. By default we let $K_{\text{max}}=12$, $n_{\text{fit}}=10$, and $n_{\text{leaf}}=5$ which are common choices for regression trees. We also ran ridge linear regression from the Scikit-Learn package \citep{scikit-learn} and the lasso \citep{tibshirani1996regression}. When running methods that were not devised to deal with categorical variables, we one-hot encoded such variables first.
 
 We randomly split each dataset into 5 folds and computed the cross-validated mean square error, that is, the average test MSE of each method on each of the 5 folds. The final score of each method is its MSE divided by the lowest MSE on that dataset. A score of 1 thus corresponds with the method that performed best on that dataset, and the scores of the other methods say how much larger their MSE was.

\subsection{Results}

The results are summarized in Table~\ref{tab:raw}. The blue entries correspond to scores of at most 1.05, that is, methods whose MSE was no more than 1.05 times that of the best MSE on that dataset. PILOT outperformed FRIED 16 times out of 19. It also did better than CART on 17 datasets, and outperformed M5 on 15. In the bottom two lines of the table we see that PILOT achieved the best average score with the lowest standard deviation.

\begin{table}[!ht]
\centering
\small
\caption{MSE ratios relative to best, for the datasets of Table~\ref{tab:datasets}.}
\begin{tabular}{lcccccccc} 
\hline
 & PILOT& Ridge& LASSO & FRIED & CART  &M5\\ 
\hline
Abalone &\textbf{\textcolor{blue}{1.00}}	&\textbf{\textcolor{blue}{1.02}}	&\textbf{\textcolor{blue}{1.01}}&	1.11&	1.12&	1.56\\
 
Admission &1.19	&\textbf{\textcolor{blue}{1.00}}	&\textbf{\textcolor{blue}{1.00}}	&1.11&	1.41&	(1.21)\\ 
 
Airfoil & 1.98	&4.40	&4.40	&\textbf{\textcolor{blue}{1.00}}&	1.75&	9.52\\

Bike &\textbf{\textcolor{blue}{1.00}}	&3.03	&3.05	&1.84	&\textbf{\textcolor{blue}{1.05}}	&1.45\\ 

Bodyfat &1.15	&\textbf{\textcolor{blue}{1.00}}	&\textbf{\textcolor{blue}{1.03}}&	1.47&	1.55&	(1.16)\\ 

Boston &\textbf{\textcolor{blue}{1.02}}	&\textbf{\textcolor{blue}{1.00}}&	1.06&	1.17&	1.16&	3.99\\ 

California & \textbf{\textcolor{blue}{1.00}}&	1.95&	1.95&	1.08&	1.42&	(11.74)\\ 

Communities &1.08&	\textbf{\textcolor{blue}{1.01}}&\textbf{\textcolor{blue}{1.00}}&	1.17&	1.32&	\textbf{\textcolor{blue}{1.04}}\\ 

Concrete &\textbf{\textcolor{blue}{1.00}}&	2.61&	2.62&	3.42&	1.38	&2.48\\ 

Diabetes & 1.07&	\textbf{\textcolor{blue}{1.00}}&	\textbf{\textcolor{blue}{1.00}}&	1.24&	1.31&	1.12\\

Electricity &\textbf{\textcolor{blue}{1.00}}&	2.75	&2.75	&1.28&	1.88&	2.19\\ 

Energy &1.17&	3.60&	3.57&	\textbf{\textcolor{blue}{1.00}}&	1.29&	1.12\\ 

Ozone & 1.36&	\textbf{\textcolor{blue}{1.05}}&	\textbf{\textcolor{blue}{1.00}}&	1.37&	1.18&	1.22\\

Residential &1.20	&1.08&	\textbf{\textcolor{blue}{1.01}}&	116.31&	1.99&	\textbf{\textcolor{blue}{1.00}}\\ 

Riboflavin &2.21	&\textbf{\textcolor{blue}{1.00}}&	1.14&	3.43&	2.89&	2.44&\\ 

Skills &\textbf{\textcolor{blue}{1.00}}&	\textbf{\textcolor{blue}{1.02}}&	\textbf{\textcolor{blue}{1.03}}	&1.06	&1.17&	(1.06)\\

Superconductor &\textbf{\textcolor{blue}{1.00}}&	2.11&	2.12& ** &** &			(9.74)\\ 

Temperature &1.07&	1.48&	1.54&	1.36&	2.19&	\textbf{\textcolor{blue}{1.00}}\\ 

Walmart & \textbf{\textcolor{blue}{1.00}} &	1.36 &	1.37 &	11.96 &	12.61 &	1.30\\

Wine &\textbf{\textcolor{blue}{1.00}}&	1.11&	1.11&	1.07&	1.08&	1.35\\

\hline
Mean & \textbf{1.18} & 1.73 & 1.74 & 8.08 & 2.09 & (2.88) \\
Std & \textbf{0.33} & 1.03 & 1.02 & 26.33 & 2.59 & (3.32) \\
\hline
\end{tabular}
\begin{center}
Each entry is the MSE of that method on that dataset, divided
by the\\ lowest MSE of that row.
Entries of at most 1.05 are shown in \textcolor{blue}{blue}.\\
**: Exceeded preset maximal wall time.\\
(xxx): M5 results based on the rescaled predictors.\\
\end{center}
\label{tab:raw}
\end{table}

Table~\ref{tab:raw} has some unusual entries that require an explanation. The column of FRIED has an extremely high test error for the \texttt{Residential} dataset. A careful examination revealed that this was caused by extrapolation. At one node, the value of the predictor of a test case was outside the range of the training cases, so \textsc{lin}  extrapolated and gave an unreasonable prediction. In the last column, the M5 results were partially adjusted because its predictions are not invariant to scaling the predictors. Therefore we ran M5 both with and without predictor scaling, and present the best result here. Those based on scaling are shown in parentheses. Unfortunately, the performance of M5 often remained poor. On the \texttt{Superconductor} data the training time for FRIED and CART exceeded the preset wall time, despite applying optimized functions (from Scikit-Learn) for cross validation and pruning. 

A closer look at the results reveals that PILOT struggled on the \texttt{Airfoil} data. A potential explanation is the fact that the \texttt{Airfoil} data has a rather complicated underlying model, which is often modeled using ensemble methods such as gradient boosting and random forests \citep{patri2015random}. The tree-based methods including PILOT gave relatively poor predictions on the \texttt{Riboflavin} data, which is high-dimensional with only 78 observations but more than 7000 predictors, and whose response has a strong linear relationship with the regressors. This explains why ridge regression and the lasso performed so well on these data. The tree-based methods suffered from the relatively small number of cases in the training set, which resulted in a higher test error. Still, PILOT did better than  CART, FRIED and M5 on these data, suggesting that it is indeed less prone to overfitting and better at capturing linear relationships.
 
There are several other datasets on which Ridge and/or Lasso do very well, indicating that the data have linear patterns. The results show that PILOT tended to do quite well on those datasets and to outperform CART on them. This confirms that the linear components of PILOT make it well suited for such data, whereas we saw that it usually did at least as well as the other tree-based methods on the other datasets.

\subsection{Results after transforming predictors}

It has been observed that linear regression methods tend to perform better if the numerical predictors are not too skewed. In order to address potentially skewed predictors, we reran the empirical experiments after preprocessing all numerical predictors by the Yeo-Johnson (YJ) transformation \citep{yeo2000new}. The parameter of the YJ transformation was fit using maximum likelihood on the training set, after which the fitted transformation parameter was applied to the test set. One way in which such a transformation could help is by mitigating the extrapolation issue as this is expected to occur mainly on the long tailed side of a skewed predictor, which becomes shorter by the transformation.

\begin{table}[!hbt]
\centering
\small
\caption{MSE ratios relative to best, after transforming the predictors.}
\begin{tabular}{lcccccccc} 
\hline
	 & PILOT & Ridge & LASSO & FRIED & CART &M5\\ 
\hline
Abalone &0.98&	1.02	&1.00	&1.21	&1.12&	1.08

\\

Admission &1.19	&1.00&	1.00	&1.12	&1.39&	1.20
\\	

Airfoil &1.66&	4.70&	4.71&	1.11	&1.78&	9.98
\\

Bike &1.01&	3.04&	3.05&	1.29	&1.05	&2.20
\\

Bodyfat&1.13	&0.96	&0.94	&1.49	&1.58&	1.16
\\

Boston & 1.06 & 	0.93 & 	0.96 & 	1.09 & 	1.12 & 	5.97
\\	

California & 1.03 & 	1.87 & 	1.87 & 	1.15 & 	1.43 & 	24.45
\\

Communities&1.05&	1.04	&1.03	&1.16	&1.30&	1.01
\\

Concrete &0.95&	1.24	&1.24	&1.23&	1.39&	2.48
\\ 

Diabetes & 1.09&	1.01&	1.01&	1.22	&1.32&	1.20
\\

Electricity&1.00	&2.75&	2.75&	1.29	&1.87& 2.71
\\

Energy&1.34&	3.84&	3.86&	1.01	&1.29	&1.11
\\

Ozone & 1.25 &	1.07 &	1.00 &	1.16 &	1.01 &	1.22
\\

Residential&0.78&	0.98&	0.81&	116.31	&2.01&	0.81
\\

Riboflavin & 1.53	&1.01	&1.28&	3.57	&3.00	&2.65

\\

Skills&0.96&	0.93&	0.93&	1.03	&1.17&	1.06
\\

Superconductor & 1.05&	2.04&	2.05&	***&	***	&9.74
\\

Temperature & 1.10& 	1.53& 	1.61	& 1.42& 	2.17	& 1.60

\\

Walmart & 1.00&	1.36&	1.38&	11.89&	11.82	&1.30
\\	

Wine & 1.01	& 1.09& 	1.09& 	1.07	& 1.08& 	1.35
\\
\hline
Mean &\textbf{1.11}&	1.67&	1.68&	7.94&	2.05	&3.71
 \\
Std & \textbf{0.20}	&1.08	&1.09	&26.36	&2.42&	5.59
 \\
\hline
\end{tabular}
\begin{center}
Each entry is the MSE of that method on that dataset,\\ divided
by the lowest MSE of the same row in Table~\ref{tab:raw}.\\
**: Exceeded preset maximal wall time.\\
\end{center}
\label{tab:transformed}
\end{table}

Table~\ref{tab:transformed} shows the results after transforming the predictors. The MSE ratios are relative to the lowest MSE of each row in Table~\ref{tab:raw}, so the entries in Table~\ref{tab:transformed} can be compared directly with those in Table~\ref{tab:raw}. We see that the transformation often enhanced the performance of PILOT as well as that of some other methods. On the high dimensional \texttt{Riboflavin} data set the score of PILOT came closer to Ridge and Lasso, whereas the other decision trees kept struggling. Also on the other data sets with linear structure PILOT was comparable with both linear methods. The average and the standard deviation of its score were reduced by the transformation as well. The overall performance of PILOT remained the best.

\section{Conclusion}\label{sec:conclusions}

In this paper we presented a new linear model tree algorithm called PILOT. It is computationally efficient as it has the same time and space complexity as CART. PILOT is regularized by a model selection rule to avoid overfitting. This regularization adds no computational complexity, and no pruning is required once the tree has been built. This makes PILOT faster than most existing linear model trees. The prediction can easily be interpreted due to its tree structure and its additive piecewise linear nature. To guide variable selection, a measure of feature importance can be defined in the same way as in CART. PILOT applies two truncation procedures to the predictions, in order to avoid the extreme extrapolation errors that were sometimes observed with other linear model trees. An empirical study found that PILOT often outperformed the CART, M5, and FRIED decision trees on a variety of datasets. When applied to roughly linear data, PILOT behaved more similarly to high-dimensional linear methods than other tree-based approaches did, indicating a better ability to discover linear structures. We proved a theoretical guarantee for its consistency on a general class of functions. When applied to data generated by a linear model, the convergence rate of PILOT is polynomial. To the best of our knowledge, this is the first linear model tree with proven theoretical guarantees.

We feel that PILOT is particularly well suited for fields that require both performance and explainability, such as healthcare \citep{healthcare}, business analytics \citep{ba, dtapp1} and public policy \citep{policy}, where it could support decision making.
 
A future research direction is to integrate PILOT trees as base learners in ensemble methods, as it is well known that the accuracy and diversity of base learners benefit the performance of the ensemble. On the one hand we have seen that PILOT gave accurate predictions on a number of datasets. On the other hand, the wide choice of models available at each node allows for greater diversity of the base learner. For these reasons and the fact that PILOT requires little computational cost, we expect it to be a promising base learner for random forests and gradient boosting.


\clearpage
\vspace{1cm}
\appendix
\begin{center}
{\bf APPENDIX}
\end{center}
\section{Preliminary results}
In this section we provide notations and preliminary results for the theorems in Section \ref{sec:theory}.
\subsection{Notation}
We follow the notations in Sections \ref{sec:description} and \ref{sec:theory}. The $n$ response values form the vector $Y$ with mean $\overline{Y}$. The values of the predictors of one case are combined in $X\in\maR^p$. The design matrix of all $n$ cases is denoted as $\bX\in\maR^{n\times p}$. Given some tree node $T$ with $t$ cases, we denote by $\bX_T:=(X_{T_1},\dots,X_{T_t})^\top\in\maR^{t\times p}$ the restricted data matrix and by $X^{(j)}_T$ its $j$-th column. The variance of the $j$-th column is given by $\hat\sigma_{j,T}^2:=\sum_{k=1}^t(X_{T_k}^{(j)}-\overline{X^{(j)}_T})^2/t$. We also let $(\hat\sigma_{j,T}^U)^2$ be the classical unbiased variance estimates with denominator $t-1$. $T$ is omitted in these symbols when $T$ is the root node or if there is no ambiguity.

We denote by $\maT_k$ the set of nodes of depth $k$ together with the leaf nodes of depth less than $k$. The PILOT prediction on these nodes is denoted as $\hat f(\maT_k)\in\mathbb R^{n}$, and $\hat f^k_T$ denotes the prediction of the selected model (e.g.\ \textsc{plin}) on some $T\in\mathcal{T}_k$\,, obtained by fitting the residuals $Y-\hat f(\maT_{k-1})$ of the previous level. The impurity gain of a model on a node $T$ with $t$ cases can be written as
$$\widehat\Delta^{k+1}(T) = ||Y - \hat f(\maT_k)||_t^2 - P(t_l)||Y - \hat f(\maT_k) - \hat f^{k+1}_{T_l}||_{t_l}^2 - P(t_r)||Y - \hat f(\maT_k) - \hat f^{k+1}_{T_r}||_{t_r}^2,$$
where $T_l$, $T_r$ are the left and right child nodes, containing $t_l$ and $t_r$ cases, and $P(t_l) = t_l/t$, $P(t_r) = t_r/t$. 

Since $\hat f(\maT_k)$ is additive and we assume that all $f\in\mathcal{F}$ are too, we can write the functions on the $j$-th predictor as $\hat f_j(\maT_k)$ and $f_j$\,. The total variation of a function $f$ on $T$ is denoted as $||f||_{TV(T)}$. For nonzero values $A_n$ and $B_n$ which depend on $n\rightarrow\infty$, we write $A_n\precsim B_n$ if and only if $A_n/B_n\le \mathcal{O}(1)$, and $\succsim$ is analogous. We write $A_n\asymp B_n$ if $A_n/B_n= \mathcal{O}(1)$. The difference of the sets $C_1$ and $C_2$ is denoted as $C_1\backslash C_2$.

\subsection{Representation of the gain and residuals}
In this section we show that the impurity gain of \textsc{lin} and \textsc{pcon} can be written as the square of an inner product. It is also shown that \textsc{plin} can be regarded as a combination of these two models. These formulas form the starting point of the proof of the consistency of PILOT and its convergence rate on linear models.

\begin{proposition}\label{p1}
Let $\hat\beta$ be the least squares estimator of a linear model fitting $Y\in \mathbb R^p$ in function of a design matrix $\bX\in\maR^{n\times p}$. Then we have
\begin{equation}
\langle Y,\bX\hat\beta \rangle_n = ||\bX\hat\beta||^2_n\;.
\end{equation}
\end{proposition}
\begin{proof}
This follows from $\hat\beta = (\bX^\top \bX)^{-1}\bX^\top Y$ and $\langle Y,\bX\hat\beta \rangle_n = Y^\top \bX\hat\beta/n$.
\end{proof}

\begin{lemma}[Representation of LIN]\label{lem:lin_pre}
Let $T$ be a node with $t$ cases. For a \textsc{lin} model on the variable $X^{(j)}$ it holds that:
\begin{equation}\label{lin_pre}
\widehat\Delta_{lin}^{k+1} = \abs*{\Bigg\langle Y-\hat f(\maT_k),\frac{X^{(j)}_T-\overline{X^{(j)}_T}}{\hat\sigma_{j,T}} \Bigg\rangle_t}^2+\left(\overline{Y}-\overline{\hat f(\maT_k)}\right)^2.
\end{equation}
Moreover, if we normalize the second term in the last expression and denote
$$ \widetilde f^{T} _{lin} := \frac{X^{(j)}-\overline{X^{(j)}_T}}{\sqrt{w(t)}\hat\sigma_{j,T}}\mathbbm 1\{X\in T\}$$
where $w(t) = t/n$, we have
$$\hat f^{k+1}_T= \langle Y-\hat f(\maT_k),\widetilde f^{T} _{lin}\rangle_n\; \widetilde f^{T} _{lin}+\overline{Y-\hat f(\maT_k)}\;.$$
\end{lemma}
\begin{proof}
By Proposition \ref{p1} we have
\begin{align*}
\widehat\Delta_{lin}^{k+1} &= ||Y - \hat f(\maT_k)||_t^2 - ||(Y - \hat f(\maT_k)) - (\hat\alpha + \hat\beta X^{(j)}_T)||_t^2\\
&=2\langle Y-\hat f(\maT_k), \hat\alpha + \hat\beta X^{(j)}_T \rangle_t - || \hat\alpha + \hat\beta X^{(j)}_T ||^2_t\\
&=\langle Y-\hat f(\maT_k), \hat\alpha + \hat\beta X^{(j)}_T \rangle_t\;.
\end{align*}
For $\hat\beta$ we don't need to take the mean of the residuals into account, hence
$$\hat\beta = \frac{(X^{(j)}_T - \overline{X^{(j)}_T})^\top(Y-\hat f(\maT_k))}{(X^{(j)}_T - \overline{X^{(j)}_T})^\top(X^{(j)}_T - \overline{X^{(j)}_T})}= \frac{\langle Y-\hat f(\maT_k), X^{(j)}_T - \overline{X^{(j)}_T}\rangle_t}{\hat\sigma_{j,T}^2}\;.$$
On the other hand $\hat\alpha = \overline{Y-\hat f(\maT_k)}-\hat\beta \overline{X^{(j)}_T}$, therefore 
\begin{align*}
\widehat\Delta_{lin}^{k+1}
&=\hat\beta\langle Y-\hat f(\maT_k),  X^{(j)} - \overline{X^{(j)}} \rangle_t+\left(\overline{Y}-\overline{\hat f(\maT_k)}\right)^2\\
&=\abs*{\Bigg\langle Y-\hat f(\maT_k),\frac{X^{(j)}-\overline{X^{(j)}}}{\hat\sigma_{j,T}} \Bigg\rangle_t}^2+\left(\overline{Y}-\overline{\hat f(\maT_k)}\right)^2\;.
\end{align*}
The second result follows from the definition of $\hat \alpha$, $\hat \beta$ and the above formula.
\end{proof}

Next we show a similar result for \textsc{pcon}. Note that here our $\hat f(\maT_k)$ is not a constant on each node, so we need to adapt the result for CART in \cite{UCDT} to our case.
\begin{lemma}[Representation of PCON]\label{lem:pcon_pre}
Let $T$ be a node with $t$ cases and $T_l$, $T_r$ be its left and right children. We then have 
$$\widehat\Delta_{pcon}^{k+1} = \abs*{\Bigg\langle Y-\hat f(\maT_k),\frac{\mathbbm 1\{X_T\in T_l\}t_r-\mathbbm1\{X_T\in T_r\}t_l}{\sqrt{t_rt_l}} \Bigg\rangle_t}^2+\left(\overline{Y}-\overline{\hat f(\maT_k)}\right)^2.$$
Moreover, if we normalize the second term in the inner product and denote
$$\widetilde f^{T} _{pcon}:=\frac{\mathbbm 1\{X_T\in T_l\}P(t_r)-\mathbbm1\{X_T\in T_r\}P(t_l)}{\sqrt{w(t)P(t_r)P(t_l)}},$$
where $w(t)=t/n$, we have
$$\hat f^{k+1}_T = \langle Y - \hat f(\maT_k), \widetilde f^{T} _{pcon} \rangle_n\;\widetilde f^{T} _{pcon}+\overline{Y-\hat f(\maT_k)}\;.$$
\end{lemma}
\begin{proof}
By the definition of \textsc{pcon}, we have
$$\hat f^{k+1}= (\overline{Y}_l - \overline{\hat f(\maT_k)}_l)\mathbbm 1(X_T\in T_l)+(\overline{Y}_r - \overline{\hat f(\maT_k)}_r)\mathbbm 1(X_T\in T_r)$$
where $\overline{Y}_l$ is the mean of $Y$ in $T_l$, $\overline{\hat f(\maT_k)}_l$ is the mean of $\hat f(\maT_k)$ in $T_l$, and similarly for the cases in the right child node. In the following, we also denote the constant predictions for the left and right node as $\hat f^{k+1}_l:=\hat f^{k+1}\mathbbm 1(X_T\in T_l)$ and $\hat f^{k+1}_r:=\hat f^{k+1}\mathbbm 1(X_T\in T_r)$. Thus we have
\begin{align*}
\widehat\Delta_{pcon}^{k+1} &= \frac{1}{t}\Big[2\sum_{X_T \in T_l}(Y-\hat f(\maT_k))\hat f^{k+1} - \sum_{X_T \in T_l}(\hat f^{k+1})^2\\
&+ 2\sum_{X_T\in T_r}(Y-\hat f(\maT_k))\hat f^{k+1} - \sum_{X_T \in T_r}(\hat f^{k+1})^2 \Big]\\
&= \frac{1}{t}[2t_l(\hat f^{k+1}_l)^2- t_l(\hat f^{k+1}_l)^2 + 2t_r(\hat f^{k+1}_r)^2 - t_r(\hat f^{k+1}_r)^2]\\
&=\frac{1}{t}[t_l\left(\overline{Y}_l - \overline{\hat f(\maT_k)}_l\right)^2 + t_r\left(\overline{Y}_r - \overline{\hat f(\maT_k)}_r\right)^2]\\
&=\frac{t_lt_r}{t^2}\Big[\left(\overline{Y}_l - \overline{\hat f(\maT_k)}_l\right)^2 +\left(\overline{Y}_r - \overline{\hat f(\maT_k)}_r\right)^2 + \frac{t_l}{t_r}\left(\overline{Y}_l - \overline{\hat f(\maT_k)}_l\right)^2+ \frac{t_r}{t_l}\left(\overline{Y}_r - \overline{\hat f(\maT_k)}_r\right)^2 \Big]\\
&\overset{(i)}{=}\frac{t_lt_r}{t^2}\Big[\left(\overline{Y}_l - \overline{\hat f(\maT_k)}_l\right)^2 + \left(\overline{Y}_r - \overline{\hat f(\maT_k)}_r\right)^2 - 2(\overline{Y}_l - \overline{\hat f(\maT_k)}_l)(\overline{Y}_r - \overline{\hat f(\maT_k)}_r)\Big]\\
&\quad\quad+\left(\overline{Y}-\overline{\hat f(\maT_k)}\right)^2\\
&=\frac{t_lt_r}{t^2}\Big[(\overline{Y}_l - \overline{\hat f(\maT_k)}_l) - (\overline{Y}_r - \overline{\hat f(\maT_k)}_r)\Big]^2+\left(\overline{Y}-\overline{\hat f(\maT_k)}\right)^2
\end{align*}
where (i) follows from the fact that
\begin{align*}
\frac{t_l}{t_r}\left(\overline{Y}_l - \overline{\hat f(\maT_k)}_l\right)^2&+ \frac{t_r}{t_l}\left(\overline{Y}_r - \overline{\hat f(\maT_k)}_r\right)^2 +2(\overline{Y}_l - \overline{\hat f(\maT_k)}_l)(\overline{Y}_r - \overline{\hat f(\maT_k)}_r)\\
&=\frac{1}{t_lt_r}\Big(\sum_{X_T\in T_l}(Y-\hat f(\maT_k))+\sum_{X_T\in T_r}(Y-\hat f(\maT_k))\Big)^2\\
&=\frac{t^2}{t_lt_r}\left(\overline{Y}-\overline{\hat f(\maT_k)}\right)^2.
\end{align*}
Therefore,
\begin{align*}
\widehat\Delta_{pcon}^{k+1} &= \abs*{\frac{1}{t}\frac{t_lt_r(\overline{Y}_l - \overline{\hat f(\maT_k)}_l)-t_lt_r(\overline{Y}_r - \overline{\hat f(\maT_k)}_r)}{\sqrt{t_lt_r}}}^2+\left(\overline{Y}-\overline{\hat f(\maT_k)}\right)^2\\
&=\abs*{\Bigg\langle Y-\hat f(\maT_k),\frac{\mathbbm 1\{X_T\in T_l\}t_r-\mathbbm1\{X_T\in T_r\}t_l}{\sqrt{t_rt_l}} \Bigg\rangle_t}^2+\left(\overline{Y}-\overline{\hat f(\maT_k)}\right)^2
\end{align*}
Finally we have
\begin{align*}
&\quad\langle Y - \hat f(\maT_k), \widetilde f^{T}_{pcon} \rangle_n\widetilde f^{T}_{pcon} + \overline{Y-\hat f(\maT_k)}\\
&= \frac{1}{t}[(\overline{Y}_l - \overline{\hat f(\maT_k)}_l) - (\overline{Y}_r - \overline{\hat f(\maT_k)}_r)](\mathbbm 1\{X_T\in T_l\}t_r-\mathbbm1\{X_T\in T_r\}t_l)+\overline{Y-\hat f(\maT_k)}\\
&=\frac{t_l+t_r}{t}[(\overline{Y}_l - \overline{\hat f(\maT_k)}_l)\mathbbm 1\{X_T\in T_l\} +(\overline{Y}_r - \overline{\hat f(\maT_k)}_r)\mathbbm 1\{X_T\in T_r\} ]\\
&\quad-\overline{Y-\hat f(\maT_k)}(\mathbbm 1\{X_T\in T_l\}+\mathbbm 1\{X_T\in T_r\})+\overline{Y-\hat f(\maT_k)}\\
& = \hat f^k_T
\end{align*}
\end{proof}

Note that \textsc{plin} is equivalent to a combination of the two preceding models (\textsc{pcon} followed by two \textsc{lin} models). As a result, we have a similar expansion for its prediction.
\begin{proposition}[Representation of PLIN]
For \textsc{plin}, there exists a function $\Tilde{f}^T_{\text{plin}}$ depending on $X$, $T$, $T_l$ and $T_r$ such that 
$$\hat f^{k+1}_T = \langle Y - \hat f(\maT_k), \widetilde f^{T} _{plin} \rangle_n\;\widetilde f^{T} _{plin}+\overline{Y-\hat f(\maT_k)}$$
\end{proposition}
\begin{proof}
As \textsc{plin} can be regarded as a \textsc{pcon} followed by two \textsc{lin} fits on the same predictor, we have
\begin{align*}
  \hat f^{k+1}_T &= \langle Y-\hat f(\maT_{k}), \widetilde f^T_{pcon} \rangle_n\; \widetilde f^T_{pcon} +  \langle Y-\hat f(\maT_k)-\hat{f}_{pcon}, \widetilde f^{T_l}_{lin} \rangle_n\;\widetilde f^{T_l}_{lin}+ \\
  &\quad +  \langle Y-\hat f(\maT_k)-\hat f_{pcon}, \widetilde f^{T_r}_{lin} \rangle_n\; \widetilde f^{T_r}_{lin}+\overline{Y-\hat f(\maT_k)}\\
  &=\langle Y-\hat f(\maT_{k}), \widetilde f^T_{pcon} \rangle_n\; \widetilde f^T_{pcon} +  \langle Y-\hat f(\maT_{k}), \widetilde f^{T_l}_{lin} \rangle_n\; \widetilde f^{T_l}_{lin} +  \langle Y-\hat f(\maT_{k}), \widetilde f^{T_r}_{lin} \rangle_n\; \widetilde f^{T_r}_{lin}\\
  &\quad +\overline{Y-\hat f(\maT_k)}
\end{align*}
where the last equation follows from $\langle 1,\widetilde f_{lin}\rangle_n=0$ and the fact that the prediction of \textsc{pcon} is constant on two child nodes. Moreover, since $\widetilde f^{T_r}_{lin}$, $\widetilde f^{T_l}_{lin}$, $\widetilde f^T_{pcon}$ are orthogonal to each other, we can deduce that
\begin{align*}
\hat f^{k+1}  &= \langle Y-\hat f(\maT_{k}), A\widetilde f^{T}_{pcon}+B\widetilde f^{T_l}_{lin}+C\widetilde f^{T_r}_{lin} \rangle_n\; (A\widetilde f^{T}_{pcon}+B\widetilde f^{T_l}_{lin}+C\widetilde f^{T_r}_{lin})\\
&\quad +\overline{Y-\hat f(\maT_k)}
\end{align*}
for some $A,B$ and $C$ depending on $X$ and the nodes $T$, $T_l$ and $T_r$.
\end{proof}
\allowdisplaybreaks 

\section{Proof of the universal consistency of PILOT}\label{app:B}
In this section we show the consistency of PILOT for general additive models. As discussed in Section \ref{sec:theory}, we begin with the estimation of the impurity gains of PILOT. Then we derive a recursion formula which allows us to develop an oracle inequality and prove the consistency at the end. For all the intermediate results, we assume the conditions in Theorem \ref{theorem:consistency} holds.

\subsection{A lower bound for \texorpdfstring{$\Delta_{pcon}$}{TEXT}}\label{app:B.1}
The key to the proof of the consistency is to relate the impurity gain $\Delta^{k+1}$ to the training error $R_k:=||Y-\hat f(\mathcal{T}_k)||^2_n-||Y-f||^2_n$ so that a recursion formula on $R_k$ can be derived. Recently, \cite{UCDT} developed such a recursion formula for CART. He showed in Lemma 7.1 that for CART, $\Delta^{k+1}$ is lower bounded by $R_k^2$, up to some factors. However, the proof used the fact that $\langle Y-\overline{Y}_t,\overline{Y}_t\rangle=0$, and this formula no longer holds when $\overline{Y}_t$ is replaced by $\hat f(\maT_k)$. The reason is that we cannot assume that the predictions in the leaf nodes of a PILOT tree are given by the mean response in that node, due to earlier nodes. To overcome this issue, we provide a generalized result for \textsc{pcon} in the following.

\begin{lemma}\label{l3}
Assuming $R_k(T)>0$ in some node $T$, then the impurity gain of \textsc{pcon} on this node satisfies:
\begin{equation*}
   \hat\Delta^{k+1}_{pcon}(\hat s,\hat j, T)\geqslant \frac{R_{k}^2(T)}{(||f||_{TV}+6B)^2}
\end{equation*}
where $\hat s$, $\hat j$ is the optimal splitting point and prediction.
\end{lemma}
\begin{proof}
Throughout this proof, the computations are related to a single node $T$ so that we sometimes drop the $T$ in this proof for notational convenience. We first consider the case where the mean of the previous residuals in the node is zero. We will show that
$$\hat \Delta^{k+1}_{pcon}(\hat s,\hat j, t)\geqslant\frac{|\langle Y-\hat f(\maT_k), f-\hat f(\maT_k)\rangle_t|^2}{4||f-\hat f(\maT_k)||^2_{TV}}\;.$$

Note that $\hat f(\maT_k)$ is the sum of $p_n$ linear functions defined on $p_n$ predictors. Let $\hat\beta_j$ be the slope of the linear function on $j$-th predictor. Then we define a probability measure $\Pi(s,j)$ on the space $\mathbb R\times \{1,\dots,p_n\}$ by the following Radon-Nikodym derivative:
$$\frac{d\Pi(s,j)}{d(s,j)} = \frac{|f'_j(s)-a_j|\sqrt{P(t_l)P(t_r)}}{\sum_{j=1}^{p_n}\int|f'_j(s')-a_j|\sqrt{P(t'_l)P(t'_r)}ds'}\;.$$
Here, $(t_l,t_r)$ and $(t'_l,t'_r)$ are the child nodes due to the split point $s$ and predictor $j$, and $a_j$ and $f'_j(x)$ are defined as 
$$a_j=\hat\beta_j,\quad f'_j(s)=
\begin{dcases}
\frac{f_j(X_{(i)}^{(j)})-f_j(X_{(i-1)}^{(j)})}{X_{(i)}^{(j)}-X_{(i-1)}^{(j)}} & X_{(i-1)}^{(j)}<s<X_{(i)}^{(j)}\\
0 & \text{otherwise}
\end{dcases}$$
where $X_{(1)}^{(j)},\dots X_{(t)}^{(j)}$ are the ordered data points along the $j$-th predictor.

The idea is that the optimal impurity gain $\hat\Delta_{pcon}(\hat s,\hat j,t)$ is always larger than the average impurity gain with respect to this probability measure: 
\begin{align*}
\hat\Delta_{pcon}(\hat s,\hat j,t)\geqslant&\int\hat\Delta(s,j,t)d\Pi(s,j)\\
=&\int |\langle Y-\hat f(\maT_k),\sqrt{w(t)}\widetilde f^{T} _{pcon}\rangle_t|^2d\Pi(s,j)\\
\geqslant & \left(\int |\langle Y-\hat f(\maT_k),\hat f_t\rangle_t|d\Pi(s,j)\right)^2
\end{align*}
where the last inequality follows from Jensen's inequality, and the representation in Lemma \ref{lem:pcon_pre} is used.
We now focus on the term inside the brackets, for which we obtain an analog of (29) in \cite{UCDT}:
   $$\int |\langle Y-\hat f(\maT_k),\hat f_t\rangle_t|d\Pi(s,j)\geqslant\frac{|\langle Y-\hat f(\maT_k),\sum_{j=1}^{p_n}\int(f'_j(s)-a_j)\mathbbm 1_{\{X^{(j)}>s\}}ds\rangle_t|}{\sum_{j=1}^p\int|f'_j(s')-a_j|\sqrt{P_{t'_l}P_{t'_r}}ds'}\,.$$
\noindent For the numerator we have
\begin{align*}
&\Big\langle Y-\hat f(\maT_k),\sum_{j=1}^{p_n}\int(f'_j(s)-a_j)\mathbbm 1_{\{X^{(j)}>s\}}ds\Big\rangle_t \\
\overset{(i)}{=} &\Big\langle
Y-\hat f(\maT_k),\sum_{j=1}^{p_n}\int(f'_j(s)-a_j)\mathbbm 1_{\{X^{(j)}>s\}}ds+\sum_{j=1}^{p_n}(f_j(X_{(1)}^{(j)})-\hat f_j(\maT_k)|_{x=X_{(1)}^{(j)}})\Big\rangle_t\\
=&\Big\langle Y-\hat f(\maT_k),\sum_{j=1}^{p_n}\Big(f_j(X^{(j)}_{(1)})+\sum_{X^{(j)}_{(2)}\leqslant X^{(j)}_{(i)}\leqslant X^{(j)}}\frac{f_j(X^{(j)}_{(i)})-f_j(X^{(j)}_{(i-1)})}{X^{(j)}_{(i)}-X^{(j)}_{(i-1)}}(X^{(j)}_{(i)}-X^{(j)}_{(i-1)})\\
&-\hat f_j(\maT_k)|_{x=X_{(1)}^{(j)}}-a_j(X^{(j)}-X^{(j)}_{(1)})\Big)\Big\rangle_t\\
=&\Big\langle Y-\hat f(\maT_k),\sum_{j=1}^{p_n}(f_j-\hat f_j(\maT_k))\Big\rangle_t\\
=&\Big\langle Y-\hat f(\maT_k),f-\hat f(\maT_k)\Big\rangle_t
 \end{align*}
where (i) again used $\langle1, Y-\hat f(\maT_k)\rangle_t=0$. 

\noindent For the denominator restricted to one predictor, we have
\begin{align*}
   &\int|f'_j(s')-a_j|\sqrt{P_{t'_l}P_{t'_r}}ds'\\=&\sum_{i=0}^{N(t)}\int_{N(t)P'(t'_l)=i}|f'_j(s')-a_j|\sqrt{(i/N(t))(1-i/N(t))}ds'\\
   =&\sum_{i=1}^{N(t)-1}\int_{X_{(i)}^{(j)}}^{X_{(i+1)}^{(j)}}|f'_{j}(s')-a_j|ds'\sqrt{(i/N(t))(1-i/N(t))}\\
   =&\sum_{i=1}^{N(t)-1}|f_{j}(X_{(i+1)}^{(j)})-a_jX_{(i+1)}^{(j)}-f_{j}(X_{(i)}^{(j)})+a_jX_{(i)}^{(j)}|\sqrt{(i/N(t))(1-i/N(t))}\\
   \leqslant&\text{ }\frac{1}{2} \sum_{i=1}^{N(t)-1}|f_{j}(X_{(i+1)}^{(j)})-a_jX_{(i+1)}^{(j)}-f_{j}(X_{(i)}^{(j)})+a_jX_{(i)}^{(j)}|\\
   =&\text{ }\frac{1}{2} \sum_{i=1}^{N(t)-1}|(f_{j}(X_{(i+1)}^{(j)})-\hat f_j(\maT_k)|_{x=X_{(i+1)}^{(j)}})-(f_{j}(X_{(i)}^{(j)})-\hat f_j(\maT_k)|_{x=X_{(i)}^{(j)} })|\\
   =&\text{ }\frac{1}{2}||f_j-\hat f(\maT_k)_j||_{TV}(T)\,. 
\end{align*}
If we sum over all $j$ predictors we have
\begin{align*}
\sum_{j'=1}^{p_n}\int|f'_j(s')-a_j|\sqrt{P_{t'_L}P_{t'_R}}ds'&\leqslant \frac{1}{2}||f-\hat f(\maT_k)||_{TV(T)}\;.
\end{align*}
\noindent Putting the bounds on the numerator and denominator together, we obtain:
\begin{align*}\int |\langle Y-\hat f(\maT_k),\hat f_t\rangle_t|d\Pi(s,j)\geqslant&\frac{|\langle Y-\hat f(\maT_k),\sum_{j=1}^{p_n}\int(f'_j(s)-a_j)\mathbbm 1_{\{X^{(j)}>s\}}ds\rangle_t|}{\sum_{j=1}^p\int|f'_j(s')-a_j|\sqrt{P_{t'_l}P_{t'_r}}ds'}\\
\geqslant& \frac{\Big\langle Y-\hat f(\maT_k),f-\hat f(\maT_k)\Big\rangle_t}{\frac{1}{2}||f-\hat f(\maT_k)||_{TV(T)}}\;.
\end{align*}
\noindent Now note that 
\begin{align*}
|\langle Y-\hat f(\maT_k),f-\hat f(\maT_k)\rangle_t|=& |\langle Y-\hat f(\maT_k),Y-\hat f(\maT_k)\rangle_t+\langle Y-\hat f(\maT_k),f-Y\rangle_t|\\
\overset{(i)}{\geqslant}& ||Y-\hat f(\maT_k)||^2_t-|\langle Y-\hat f(\maT_k),f-Y\rangle_t|\\
\overset{}{\geqslant}&||Y-\hat f(\maT_k)||^2_t-||Y-f||_t||Y-\hat f(\maT_k)||_t\\
\overset{}{\geqslant}&\frac{1}{2}||Y-\hat f(\maT_k)||^2_t-\frac{1}{2}||Y-f||^2_t\\
=&  \frac{1}{2}R_k(t) \overset{(ii)}{>}0
\end{align*}
where (i) is the triangle inequality and (ii) follows from the assumption that  $R_k(T)$ is strictly positive. 

Finally, we deal with $\frac{1}{2}||f-\hat f(\maT_k)||_{TV(T)}$ for which we need an upper bound. We have $\frac{1}{2}||f-\hat f(\maT_k)||_{TV(T)}\, \leqslant \frac{1}{2}||f||_{TV}+\frac{1}{2}||\hat f(\maT_k)||_{TV(T)}\,$. It thus remains to bound the total variation of $\hat f(\maT_k)$ on the node $T$. We now show that the sum of the total variations along all the predictors is bounded by $(\max\hat f(\maT_k)-\min\hat f(\maT_k))\leqslant 6B$. Since $\hat f(\maT_k)$ is linear, the maximum and minimum values are always attained on vertices of the node (or cube). We can therefore assume without loss of generality that the cube is $[0,1]^{p_n}$ and $\hat f(\maT_k)=\sum_{j=1}^{p_n}\hat\beta_jX^{(j)}$ s.t.\ for $\forall j<j'$, $\hat\beta_j<\hat\beta_{j'}$ and $\hat\beta_j<0\Leftrightarrow j<p_0$. Now we start with the vertex $v_0\in[0,1]^{p_n}$ such that its $j$-th variable is $1$ if and only if $j<p_0$ (otherwise $0$). Then we move to another vertex $v_1$ (along the edge) which is identical to $v_0$ except on the first entry, i.e.\ its first entry is 0. In addition, we have $||\hat f_1(\maT_k)||_{TV}\,=\hat f(\maT_k)|_{X = v_1}-\hat f(\maT_k)|_{X = v_0}\,$. Similarly, for $\forall j$, we let the vertex $v_{j+1}$ be identical to $v_{j}$ except for the $j$-th variable (by changing $1$ to $0$ or $0$ to $1$) and we have $||\hat f_j(\maT_k)||_{TV}\,=\hat f(\maT_k)|_{X=v_j}-\hat f(\maT_k)|_{X=v_{j-1}}\,$. Furthermore, it holds that $\sum_{j=1}^{p_n}||\hat f(\maT_k)_j||_{TV}\,=\max \hat f(\maT_k)-\min\hat f(\maT_k)\leqslant 6B$. Therefore, we obtain $\left(\frac{1}{2}||f-\hat f(\maT_k)||_{TV(T)}\right)^2 \leqslant  \frac{1}{4}(||f||_{TV}+6B)^2$.

Putting everything together yields
$$\hat \Delta^{k+1}_{pcon}(\hat s,\hat j, t)\geqslant\frac{|\langle Y-\hat f(\maT_k), f-\hat f(\maT_k)\rangle_t|^2}{\frac{1}{4}||f-\hat f(\maT_k)||^2_{TV}}\geqslant \frac{R_k^2(T)}{(||f||_{TV}+6B)^2}.$$

Finally, we treat the general case where the mean of the residuals in the node is not necessarily zero (which can happen  after fitting a \textsc{blin} model). Note that $R_k(T)=\{||Y-\hat f(\maT_{k})-(\overline{Y}-\overline{\hat f(\maT_k)})||^2_t-||Y-f||^2_t\}+(\overline{Y}-\overline{\hat f(\maT_k)})^2:=R_{k(1)}+R_{k(2)}$, where $R_{k(1)}$ is the squared training error after regressing out the mean and $R_{k(2)}$ is the squared mean. 
Since the first truncation procedure ensures that $R_k(T)\leqslant16B^2$, we have
$(||f||_{TV}+6B)^2\geqslant36B^2\geqslant2R_k(T)=2(R_{k(1)}+R_{k(2)})$, so that $$\frac{R_{k(1)}^2}{(||f||_{TV}+6B)^2}+R_{k(2)}\geqslant\frac{(R_{k(1)}+R_{k(2)})^2}{(||f||_{TV}+6B)^2}.$$
Thus, by the preceding results and the fact that \textsc{pcon} makes the mean of the residuals zero, we have
$$\Delta_{pcon}^{k+1}(\hat s, \hat j, t)\ge\frac{R_{k(1)}^2}{(||f||_{TV}+6B)^2}+R_{k(2)}\geqslant\frac{R_k(T)^2}{(||f||_{TV}+6B)^2}\,.$$
\end{proof}

\subsection{Proof of Proposition \ref{prop:bic}}
As \textsc{plin} generalizes \textsc{pcon}, we always have $\hat\Delta^k_{plin}\geqslant\hat\Delta^k_{pcon}$. For other models, however, it is not possible to develop a lower bound by similarly constructing a carefully designed probability measure as that in Lemma \ref{l3}. This is because we do not have an indicator function which naturally generates $f$ when associated with $f'$ in the integral (see the estimation of the numerator). Fortunately, we can rely on the BIC criterion to obtain bounds on the relative impurity gains between the different models. More precisely, Proposition \ref{prop:bic} shows that if \textsc{blin} or \textsc{lin} are chosen, their gain has to be comparable (i.e., differ by at most a constant factor) to that of \textsc{pcon}. This ensures no underfitting occurs. Moreover, if \textsc{con} is chosen at one node, all the subsequent models on that node would also be \textsc{con} which justifies the use of the \textsc{con} stopping rule. The proof is given in the following.\\

\begin{proof}
Without loss of generality, we may assume no model fits perfectly. Let us start with the first claim. By the assumption that $BIC_1>BIC_2$ we have for any $v_1,v_2$ that
\begin{align}\label{p2e1}
BIC_2 &< BIC_1\nonumber\\
t\log\left(\frac{R_0 - t\Delta_2}{t}\right)  + v_2 \log(t) &< t\log\left(\frac{R_0 - t\Delta_1}{t}\right) + v_1 \log(t)\nonumber\\
\log\Big(1-\frac{t\Delta_2}{R_0}\Big)-\log\Big(1-\frac{t\Delta_1}{R_0}\Big)&<(v_1-v_2)\frac{\log t}{t}\nonumber\\
\Big(1-\frac{t\Delta_2}{R_0}\Big)\Big/\Big(1-\frac{t\Delta_1}{R_0}\Big)&<\exp\Big((v_1-v_2)\frac{\log t}{t}\Big)\nonumber\\
\frac{t\Delta_2}{R_0}&> 1+ \exp\Big((v_1-v_2)\frac{\log t}{t}\Big)\Big(\frac{t\Delta_1}{R_0}-1\Big)
\end{align}
Therefore, for any model $i$ that is selected over \textsc{con} by the BIC criterion, we have
\begin{equation}\label{p2e2}
\frac{t\Delta_i}{R_0}> 1- \exp\Big((1-v_i)\frac{\log t}{t}\Big)    
\end{equation}
Since $t\geqslant2$, $\log(t)/t$ is always positive, so is the lower bound in the preceding equations.\\

We now proceed with the second claim. If (\ref{p2e2}) holds for model 1, we have by (\ref{p2e1}),
\begin{align*}
  \frac{\Delta_2}{\Delta_1}&>\Big(1-\exp\Big((v_1-v_2)\frac{\log t}{t}\Big)\Big)\frac{R_0}{t\Delta_1}+\exp\Big((v_1-v_2)\frac{\log t}{t}\Big)\\
  &> \Big(1-\exp\Big((v_1-v_2)\frac{\log t}{t}\Big)\Big)\Big/\Big(1-\exp\Big((1-v_1)\frac{\log t}{t}\Big)\Big) + \exp\Big((v_1-v_2)\frac{\log t}{t}\Big)\\
  &=\frac{1-\exp((1-v_2)\log t/t)}{1-\exp((1-v_1)\log t/t)}\\
  &:=C^*\left(\frac{\log t}{t}\right)> 0
\end{align*}
On the other hand, if (\ref{p2e2}) does not hold for model 1, we still get ${\Delta_2}> C^*{\Delta_1}$ by using (\ref{p2e1}) for model 2 and \textsc{con}, its inverse inequality for model 1 and \textsc{con}, and the fact that function $f(x)=(1+ax)/(1+bx)$ is monotonically increasing for $x\in[-1,0]$ and $1>a>b>0$.

Next we find the minimum of $C^*$. We compute the derivative of the numerator and denominator of $C^*$ with respect to $\log t/t$ to get $(v_2-1)\exp((1-v_2)\log t/t)$ and $(v_1-1)\exp((1-v_1)\log t/t)$. Therefore, for any $\log t/t:= s\in(0,1/2]$ we have
\begin{align*}
C^*(s)&=\frac{\int_0^s(v_2-1)\exp((1-v_2)r)dr}{\int_0^s(v_1-1)\exp((1-v_1)r)dr}\\
&\geqslant \frac{(v_2-1)\int_0^s\exp((1-v_1)r)dr}{(v_1-1)\int_0^s\exp((1-v_1)r)dr}=\frac{v_2-1}{v_1-1}
\end{align*}
due to $v_1\geqslant v_2$. Therefore we conclude that $C^*\geqslant (v_2-1)/(v_1-1)$ for any $t\geqslant 2$.

By previous lemmas we know that the impurity gain of the models can be divided into two parts. The first is from regressing out the mean and the second is from the model assumption, which does not depend on the mean of the response variable. Therefore, we can let $\Delta_2=\Delta_1+\Delta_2'$ in (\ref{p2e1}) when model 1 is \textsc{con}. Here, $\Delta_2'$ is the gain after regressing out the mean. Now, if \textsc{con} is better, we have by the inverse inequality of (\ref{p2e1}) that
\begin{align}\label{p2e3}
\nonumber\frac{t\Delta_2'}{R_0}+\frac{t\Delta_1}{R_0}&<1- \exp\Big((1-v_2)\frac{\log t}{t}\Big)\Big)\Big(1-\frac{t\Delta_1}{R_0}\Big)\\
\nonumber\Longleftrightarrow\frac{t\Delta_2'}{R_0}&< \Big(1- \exp\Big((1-v_2)\frac{\log t}{t}\Big)\Big)\Big(1-\frac{t\Delta_1}{R_0}\Big)\\\
&< 1- \exp\Big((1-v_2)\frac{\log t}{t}\Big),
\end{align}
which means that the subsequent node still prefers \textsc{con}, even if its gain is 0.
\end{proof}

As an example, for our choice of degrees of freedom, $\Delta_{lin}\geqslant\Delta_{pcon}/4$ if \textsc{lin} is chosen. Similarly, if \textsc{blin} is the preferred model, we must have $\Delta_{blin}\geqslant\Delta_{pcon}$ since their degrees of freedom are the same.

\subsection{Proof of Theorem \ref{theorem:recursion1}}
The remaining issue is that the RSS in \textsc{con} nodes does not get improved as the depth increases, since subsequent nodes will select \textsc{con} again. Therefore, we first construct a recursion formula which excludes the terms corresponding to the RSS in the \textsc{con} nodes. Then we will show that the training error in these \textsc{con} nodes vanishes asymptotically, which justifies the \textsc{con} stopping rule.

Recall the definitions of  $C_k^+=\{T|T=\textsc{con}, T\in \maT_{k-1}, R_{k-1}(T)>0\}$, the set of nodes on which \textsc{con} is fitted before the $k$-th step, and $C_k^-$ for those with $R_k(T) \leqslant  0$. Further define $C_k^*:=C_k^+\backslash C_{k-1}^+$ and $C_k^{\#}:=C_k^-\backslash C_{k-1}^-$. Finally, let $A_k^+:=\{T|T\in \maT_{k},T\neq\textsc{con},R_k(T)>0\}$ and $A_k^-:=\{T|T\in \maT_{k},T\neq\textsc{con},R_k(T)\leqslant0\}$. Note that with these disjoint sets of nodes we now have 
$$\maT_{k -1} = C_{k -1}^+ \cup C_{k -1}^- \cup A_{k -1}^+ \cup A_{k -1}^- \cup C_{k}^* \cup C_{k}^{\#}\;.$$
\noindent Next we can calculate the errors in these sets of nodes. Let $R_{C_k^+}:=\sum_{T\in C_k^+}w(T)R(T)$ and define $R_{C_k^-}$, $R_{C_k^*}$, $R_{C_k^{\#}}$, $ R_{A_{k-1}^+}$, $R_{A_{k-1}^-}$ in similar fashion. This yields
$$R_{k-1} = R_{C_{k-1}^+} + R_{C_{k-1}^-} + R_{A_{k-1}^+} + R_{A_{k-1}^-} + R_{C_k^{\#}} + R_{C_k^{*}}\;.$$
Finally, we let $\widetilde{R}_k:=R_k-R_{C_k^+}$.

Throughout the proof we assume without loss of generality that the gain of the \textsc{con} node fitted at depth k is already included into $R_k$, therefore $R_{C^+_k}$ can be regarded as the remaining error after regressing out the mean for \textsc{con} nodes before depth $k$.\\

\begin{proof}
Without loss of generality we may assume that $R_{C_k^+} \leqslant R_k$ so that $R_k>0$. Otherwise the claim follows immediately. We have
\begin{align*}
\widetilde R_k&=R_k-R_{C_k^+}\\
&\leqslant R_{k-1}-R_{C_k^+}-\sum_{T\in A_{k-1}^+}w(T)\Delta(T)\\
&\leqslant\widetilde R_{k-1}-R_{C_k^*}-\sum_{T\in A_{k-1}^+}\frac{w(T)}{F}R^2_{k-1}(T)\\
&\overset{(i)}{\leqslant}\widetilde R_{k-1}-R_{C_k^*}-\frac{1}{F}\Big(\sum_{T\in A_{k-1}^+}{w(T)}R_{k-1}(T)\Big)^2\\
&=\widetilde R_{k-1}-R_{C_k^*}-\frac{1}{F}\Big(R_{k-1}-R_{C_{k-1}^+}-R_{C_{k-1}^-}-\sum_{T\in A_{k-1}^-}w(T)R_{k-1}(T)-R_{C_k^{\#}}-R_{C_k^*}\Big)^2\\
&=:\widetilde R_{k-1}-R_{C_k^*}-\frac{1}{F}\Big(R^*_{k-1}-R_{C_k^*}\Big)^2
\end{align*}
where the second inequality follows from the preceding Lemma and Proposition, and (i) is Jensen's inequality. Since $R_{k-1}-R_{C_{k-1}^+}$ is positive and the other three terms in $R^*_{k-1}$ are negative, we may replace $R^*_{k-1}$ by $R_{k-1}-R_{C_{k-1}^+}=\widetilde R_{k-1}>0$ on the right hand side of the above inequality. Furthermore, we have by definition, 
\begin{equation}\label{recineq}
-R_{C_k^*}-(\widetilde R_{k-1}-R_{C_k^*})^2/F\leqslant\begin{dcases}
  -\widetilde R_{k-1}+F/4<-\widetilde R_{k-1}/2 &\text{if }\widetilde R_{k-1}>F/2\\
-\widetilde R_{k-1}^2/F & \text{otherwise.}
\end{dcases}    
\end{equation}

In fact, by assuming that the first regression model is not \textsc{con}, we already have $\widetilde R_1\leqslant F/4$ for any initial $R_0$ by either inequality in (\ref{recineq}). Therefore, all the subsequent estimations follow the second inequality and we have for any $k\geqslant 2$,
\begin{equation*}
\widetilde R_k\leqslant \widetilde R_{k-1}-\frac{1}{F}\widetilde R_{k-1}^2.    
\end{equation*}
By a similar induction argument to Lemma 4.1 of \cite{UCDT}, we get the estimation of $R_K$ after moving $R_{C^+_K}$ to the right hand side. 

It remains to control the errors in the \textsc{con} nodes. We can assume that the \textsc{con} models have regressed out the mean in each $C_k^+$. The first step is to use Proposition \ref{prop:bic} and Lemma \ref{l3} to get the following upper bound for any \textsc{con} node $T$:
\begin{equation}\label{conbound}
    \frac{R(T)^2}{FR_0}\leqslant \frac{\Delta_{\text{pcon}}}{R_0}\leqslant\frac{1}{t}\Big(1-\exp\Big((1-v_{pcon})\frac{\log t}{t}\Big)\Big)
\end{equation}
where we used Lemma \ref{l3} for the first inequality and (\ref{p2e3}) for the second.
Since we assume the response is almost surely bounded by $\pm B$, we have
$$R(T)\leqslant \sqrt{B^2F}\sqrt{1-\exp\Big((1-v_{pcon})\frac{\log t}{t}\Big)}\precsim\sqrt{\frac{(v_{pcon}-1)\log t}{t}}$$

\noindent Therefore the weighted sum of $R(T)$ on $C_K^+$ is asymptotically
$$\sum_{T\in C_K^+} w(T)R(T) \precsim \frac{1}{n}\sum_{T\in C_K^+} \sqrt{t\log t}\;.$$
By the Cauchy-Schwarz inequality we obtain
$$\sum_{T\in C_K^+}\sqrt{t\log t}\leqslant \sqrt{2^K}\Big({\sum_{T\in C_K^+} t\log t}\Big)^{1/2}\leqslant\sqrt{2^K}\Big({\log n\sum_{T\in C_K^+} t}\Big)^{1/2}=\sqrt{2^Kn\log n}\;.$$
Therefore if we let $K=\log_2(n)/r$ with $r>1$, $R_{C_K^+}$ is of order $\mathcal{O}(\sqrt{\log (n)/n^{(r-1)/r}})$, which tends to zero as the number of cases goes to infinity. 
\end{proof}

\subsection{Proof of Theorem \ref{theorem:consistency}}
Now we can apply Theorem \ref{theorem:recursion1} to derive an oracle inequality for our method and finally prove its consistency. We will use Theorem 11.4 together with Theorem 9.4 and Lemma 13.1 from \cite{distribution}, since these results do not require the estimator to be an empirical risk minimizer. \\

\begin{proof}
Let $f$ be the true underlying function in $\mathcal{F}$, and $\mathcal{F}_n$ be the class of the linear model tree. We first write the $L^2$ error as $||f-\hat f(\maT_K)||^2=E_1+E_2$ to apply Theorem 11.4, where
$$E_1:=||f-\hat f(\maT_K)||^2-2(||Y-\hat f(\maT_K)||^2_n-||Y-f||^2_n)-\alpha-\beta$$
and 
$$E_2:=2(||Y-\hat f(\maT_K)||_n^2-||Y-f||_n^2)+\alpha+\beta.$$
By our assumption, we can define $B_0=\max\{3B,1\}$ so that the class of functions and $Y$ are a.s.\ bounded by $[-B_0,B_0]$, which fulfills the condition of Theorem 11.4. Thus by Theorem 11.4 with $\epsilon=1/2$ in their notation, it holds that
$$P(\{\exists \hat f({\maT_K})\in{\mathcal{F}_n}\text{ s.t. } E_1>0\})\leqslant 14\sup_{\mathbf x^n}\mathcal{N}\Big(\frac{\beta}{40B_0},\mathcal{F}_n,L_1(\nu_{\mathbf {x}^n})\Big)\exp\Big(-\frac{\alpha n}{2568B_0^4}\Big)$$
where $\alpha,\beta\rightarrow 0 $ as $n\rightarrow\infty$.

Theorem 9.4 gives the estimation for the covering numbers of $\mathcal{G}_n$ which denotes the class of functions in the leaf nodes. We note that the condition of the theorem $\beta/40B_0<B_0/2$ is automatically satisfied for sufficiently large $n$ if $\beta\rightarrow 0$ is well-defined.

Combining Theorem 9.4 and Lemma 13.1, we get the estimation for the covering number of $\mathcal{F}_n$. Specifically we have,
\begin{align*}
\mathcal{N}\Big(\frac{\beta}{40B_0},\mathcal{F}_n,L_1(\nu_{\boldsymbol x^n})\Big)&\leqslant \Gamma_n(\Lambda_n)\Big[3\Big(4eB_0\frac{40B_0}{\beta}\log\Big(6eB_0\frac{40B_0}{\beta}\Big)\Big)^{V(\mathcal{G}_n)}\Big]^{2^K}\\
&\leqslant\Gamma_n(\Lambda_n)\Big[3\Big(\frac{160eB_0^2}{\beta}\log\frac{240eB_0^2}{\beta}\Big)^{V(\mathcal{G}_n)}\Big]^{2^K}\\
&\leqslant(np)^{2^K}\Big[3\Big(\frac{160eB_0^2}{\beta}\log\frac{240eB_0^2}{\beta}\Big)^{p+1}\Big]^{2^K}.
\end{align*}
Here $\Lambda_n$ is the set of all possible binary trees on the training set of $n$ cases, and $\Gamma_n(\Lambda_n)$ is the upper bound for the number of different partitions on that set that can be induced by binary trees. The upper bound on $\Gamma_n(\Lambda_n)$ follows from \cite{UCDT} and \cite{CoRF}. The exponent $V(\mathcal{G}_n)$ is the VC dimension of $\mathcal{G}_n$, and since we have multivariate linear predictions $V(\mathcal{G}_n) = p+1$. 

If we further let $\beta \asymp \frac{B_0^2}{n}$ and $\alpha \asymp 3B_0^4\log(np\log n)2^{K+\log (p+1)}/n$ we have
$$
\mathcal{N}\Big(\frac{\beta}{40B_0},\mathcal{F}_n,L_1(\nu_{\boldsymbol x^n})\Big)\precsim(np)^{2^K}(n\log n)^{2^{K+\log(p+1)}}
$$
and
\begin{align*}
P(E_1>0)&\leqslant P(\{\exists \hat f({\maT_K})\in{\mathcal{F}_n}\text{ s.t. } E_1>0\})\\
&\precsim (np)^{2^K}(n\log n)^{2^{K+\log(p+1)}}\frac{1}{(np\log n)^{3\times2^{K+\log(p+1)}}}\\
&\precsim \frac{1}{(np\log n)^{2^{K+\log(p+1)}}}\\
&\leqslant \mathcal{O}\Big(\frac{1}{n}\Big).
\end{align*}
Moreover, $E_1\leqslant C_1$ for some constant $C_1$ by the fact that $||f||_{TV}\;\leqslant A$ and both $Y$ and $f(\maT_k)$ are bounded in $[-B_0,B_0]$ for any $k$. Thus $\maE[E_1]\leqslant C_1/n$ for some $C_1>0$. By Theorem \ref{theorem:recursion1} we also have for any $f\in\mathcal{F}$ that
\begin{align*}
\maE[E_2]&\leqslant\frac{2F}{K+3}+\frac{C_2\sqrt{2^Kn\log n}}{n}+\alpha+\beta,
\end{align*}
and therefore summing up everything we have,
\begin{align}\label{oracle}
\maE[||f-\hat f(\maT_k)||^2]\leqslant &\frac{2F}{K+3}+\frac{C_3\sqrt{2^Kn\log n}}{n}+\frac{C_4\log(np\log n)2^{K+\log(p+1)}}{n}
\end{align}
where $C_3$ and $C_4$ only depend on $A$ and $B$.

By our assumption, the expected error tends to zero. Moreover, if we pick $K_n=\log(n)/r$ for some $r>1$ and $p_n=n^{s}$ such that $0<s<1-1/r$, the first term in~\eqref{oracle} is $\mathcal{O}(1/\log(n))$ and the second and third terms turn out to be $o(1/\log(n))$. Therefore, the overall convergence rate becomes $\mathcal{O}(1/\log(n))$.
\end{proof}

\section{Convergence rate for linear model data}\label{app:C}
\subsection{Preliminaries and Ideas}
In the previous section we derived the consistency of PILOT in case the underlying is a general additive function. The convergence rate obtained in that setting is $\mathcal{O}(1/\log(n))$, the same as CART but rather slow. Of course this is merely an upper bound, and it is not unlikely that we can obtain much faster convergence if the true underlying function is somehow well-behaved.

In this section, we consider such a scenario. In particular, we aim to show a faster convergence rate of PILOT for linear functions. In order to tackle this problem, we cannot use the same approach as for the general consistency result. In particular, a counterpart of Lemma \ref{l3} for the gain of \textsc{lin} would still result in a squared error in the left hand side which would restrain the convergence rate. Therefore, we take a different approach and turn to a recent tool proposed by \cite{NPBLR} in the context of boosting models. In our case, we consider a linear regression of $Y_T$ on $\boldsymbol X_T$ in each node $T$ and estimate the gradient of a related quadratic programming (QP) problem, thereby estimating the impurity gain. 

Throughout the proof we use the same notation as that in \cite{NPBLR} to indicate a subtle difference in the excess error compared to \cite{UCDT}. We first define the notation of the relevant quantities related to the loss. Let $L^*_n(T):=\min_{\hat\beta}||Y-\boldsymbol X_T\hat\beta||^2_T$ be the least squares loss on the node $T$ and $L^*_n:=\min_{\hat\beta}||Y-\boldsymbol X\hat\beta||^2_n$ be the least squares loss on the full data. We further denote the loss of a $k$ depth PILOT by $L^k_n:=||Y-\hat f(\maT_k)||^2_n$. Finally, $L^k_n(T):=||Y_T-\boldsymbol X_T\beta_T||^2_t$ is its loss in node $T$ for some $\beta_T$ (we can write the loss like this, since on $T$ the prediction function is linear). When the depth $k$ is not important, we omit the superscript. 

In the following we use the notation  $\boldsymbol{\tilde X}_T$ for the standardized predictor matrix obtained by dividing each (non-intercept) column $j$ in $\bX$ by $\sqrt{n}\hat\sigma^U_j$ where $\hat\sigma^U_j$ is the usual estimate of the standard deviation of column $j$. We then consider $L_n(T):=||Y_T-\boldsymbol{\tilde X}_T\tilde\beta_T||^2_t$ and write its gradient as $\nabla L_n(T)|_{\tilde\beta = \tilde\beta_T}\,$.

Theorem 2.1 of \cite{NPBLR} used a fixed design matrix so that a global eigenvalue condition can be imposed on it. As we are fitting linear models locally, we need a more flexible result. To this end we will impose a distributional condition that makes use of the following matrix concentration inequality:

\begin{theorem}[Theorem 1.6.2 of \cite{matrix_ineq}]\label{matrixBernstein}
Let $\boldsymbol S_1,\dots,\boldsymbol S_n$ be independent, centered random matrices with common dimension $d_1\times d_2$, and assume that $\maE[\boldsymbol S_i]=0$ and $||\boldsymbol S_i||\leqslant L$. Let $\boldsymbol Z = \sum_{k=1}^n\boldsymbol S_i$ and define
$$v(\boldsymbol Z)=\max\{||\maE[\boldsymbol Z\boldsymbol Z^\top]||,||\maE[\boldsymbol Z^\top\boldsymbol Z]||\}.$$
Then
$$P[||Z||\geqslant t]\leqslant(d_1+d_2)\exp\Big(\frac{-t^2/2}{v(\boldsymbol Z)+Lt/3}\Big)\;\;\;\text{ for all }t\geqslant 0.$$
\end{theorem}

\subsection{A probabilistic lower bound for \texorpdfstring{$\Delta_{lin}$}{TEXT}}
Our first step is to show a probabilistic bound for the impurity gain of $\Delta_{lin}$ based on the ideas in the preceding section.

\begin{lemma}\label{l4}
Let $T$ be a node with $n$ cases. We define $L_n(T)$ as the training loss of PILOT and $L^*_n(T)$ as that of least squares regression. In addition to conditions 1 and 2, we assume that the residuals of the response on $T$ have zero mean. Then there exists a constant $C_{}$ such that the \textsc{lin} model satisfies
 $$P\Bigg[\Delta_{\text{lin}}\geqslant \frac{\lambda_0(L_n(T)-L^*_n(T))}{4p}\Bigg]
 \geqslant 1-\exp(C_{\lambda_0,\sigma_0,p}n).$$
\end{lemma}
\begin{proof}
By the definition of the gradient we have,
\begin{align}\label{gtog}
\begin{split}
    \frac{n||\nabla L_n(T)|_{\tilde\beta = \tilde\beta_T}||_{\infty}}{2}&=||\tilde {\boldsymbol X}^\top r||_\infty=\max_{j\in1\dots p}\{|r^\top \tilde X^{(j)}|\}\\
    &=\max_{j\in1\dots p}\frac{|r^\top X^{(j)}|}{\sqrt n\hat\sigma_j^U}\\
    &\leqslant \max_{j\in1\dots p}\frac{|r^\top (X^{(j)}-\overline X^{(j)})|}{\sqrt n\hat\sigma_j}\\
    &=\sqrt{n\Delta_{lin}}
\end{split}
\end{align}
where $r$ denotes the residuals $Y-\tilde \bX\tilde\beta_T$ and the last equality follows from (\ref{lin_pre}) and the assumption that the residuals of the response on $T$ have zero mean.

By letting $h(\cdot)=L_n(\cdot)$, $Q=\frac{2}{n}\tilde{\bX}^\top\tilde{\bX}$ in Proposition A.1 of \cite{NPBLR} and (\ref{gtog}), we have (assuming the difference in the loss is positive),
\begin{align*}
 P\Bigg[\Delta_{\text{lin}}\geqslant \frac{\lambda_0(L_n(T)-L^*_n)}{4p}\Bigg]&
  \geqslant P\Bigg[\frac{n||\nabla L_n|_{\tilde\beta = \tilde\beta_T}||^2_{\infty}}{4}\geqslant \frac{\lambda_0(L_n(T)-L^*_n)}{4p}\Bigg]\\
 &\geqslant P\Bigg[\frac{n}{p}||\nabla L_n|_{\tilde\beta = \tilde\beta_T}||_{2}^2\geqslant \frac{\lambda_0(L_n(\tilde \beta_T)-L^*_n)}{p}\Bigg]\\
&\geqslant P[\lambda_{min}(\tilde {\boldsymbol X}^\top\tilde {\boldsymbol X})>\lambda_0]\,.
\end{align*}
All that is left to show is that 
\begin{equation*}
 P[\lambda_{min}(\tilde {\boldsymbol X}^\top\tilde {\boldsymbol X})>\lambda_0]
 \geqslant 1-\exp(-C_{\lambda_0,\sigma_0,p}n).
\end{equation*}
In order to do so, we compute the minimum eigenvalue of $\tilde {\boldsymbol X}^\top \tilde {\boldsymbol X}$. We define the vector $S_i\in\mathbb R^p$ by its coordinates $S_i^{(j)}=X_i^{(j)}/\sigma_j$ where $\sigma_j$ is the true standard deviation of the $j$-th predictor. By the properties of eigenvalues and singular values, we can split the estimation into three parts 
\begin{align*}
 \lambda_{min}(\maE[SS^\top])&=\sigma_{min}(\maE[SS^\top])\\
 &=\sigma_{min}\Big[\tilde {\boldsymbol X}^\top\tilde {\boldsymbol X}+\underbrace{\frac{1}{n}\sum_{i=1}^n S_iS_i^\top-\tilde {\boldsymbol X}^\top\tilde {\boldsymbol X}}_{\tilde A_n}+\underbrace{\frac{1}{n}\sum_{i=1}^n\Big(\maE[SS^\top]-S_iS_i^\top}_{\tilde B_n}\Big)]\\
 &\leqslant\lambda_{min}(\tilde {\boldsymbol X}^\top\tilde {\boldsymbol X})+\sigma_{max}(\tilde{A}) + \sigma_{max}(\tilde{B})
\end{align*}

Therefore, we may lower bound the eigenvalues of the matrix $\tilde {\boldsymbol X}^\top\tilde {\boldsymbol X}$ using the singular values of 3 other matrices. The left hand side is already bounded by $2\lambda_0$ by our assumption. For $\sigma_{max}(\tilde A_n)$, we note that the two matrices inside differ only in the estimation of the standard deviation. If we can show that the difference in the entries are uniformly bounded by some constant, then we could also upper bound its largest singular value. To do this, we use Theorem 10 in \cite{ebb}, which states that for i.i.d. random variables $X_1,\dots,X_n$ in [0,1] and $\delta>0$, the classical unbiased variance estimates $(\hat\sigma^U)^2$ and the true variance $\sigma^2$ satisfy
$$P\Bigg[|\sigma-\hat\sigma^U|>\sqrt{\frac{2\log(1/\delta)}{n-1}}\Bigg]<2\delta\;.$$
Due to
\begin{align*}
    \frac{1}{n}\sum_{i=1}^n S_iS_i^\top-\tilde {\boldsymbol X}^\top\tilde {\boldsymbol X}&=\frac{1}{n}\sum_{i=1}^n[S_iS_i^\top-\tilde { X_i}\tilde { X_i}^\top]\\
    &=\Bigg[\frac{\sum_{i=1}^nX^{(k)}_iX^{(l)}_i}{n}\Big(\frac{1}{\sigma_k\sigma_l}-\frac{1}{\hat\sigma^U_k\hat\sigma^U_l}\Big)\Bigg]_{kl}
\end{align*}
we only need to control the difference in the variance term, as each $X_i^{(k)}$ is bounded by 1. If we put $\sqrt{2\log(1/\delta)/(n-1)}=\lambda_0\sigma^4_0/2p$ and let $E_k$ be the event that $|\sigma_k-\hat\sigma^U_k|\leqslant\frac{\lambda_0\sigma^4_0}{2p}$, we have 
$$P[E_k]\geqslant 1-\exp(-C_{\lambda_0,\sigma_0,p}n)\;\;\;\text{ and }\;\;\;\sigma_k>\sigma_0\text{ on }E_k$$
due to $\sigma_0<1$, $\lambda_0<2\lambda_0\leqslant tr(\maE[(S-\overline{S})(S-\overline{S})^\top])\leqslant p$ and our assumption that the estimated variance should be lower bounded by $2\sigma_0^2$. On $E:=\bigcup_k E_k$ we have  
$$\max_{k,l}(|\tilde A_{kl}|)\leqslant\max_{k,l}\frac{|\sigma_k\sigma_l-\hat\sigma_k^U\hat\sigma_l^U|}{\sigma_k\sigma_l\hat\sigma_k^U\hat\sigma_l^U}\leqslant\max_{k,l}\frac{\sigma_k(|\sigma_l-\hat\sigma_l^U|)+\hat\sigma_l^U(|\sigma_k-\hat\sigma_k^U|)}{2\sigma_0^4}\leqslant\frac{\lambda_0}{2p}$$
where the last inequality is due to $\sigma_k, \sigma_l^U\leqslant1$. Thus $P[\max_{k,l}(|\tilde A_{kl}|)\leqslant\lambda_0/(2p)]\geqslant 1-p^2\exp(-C_{\lambda_0,\sigma_0,p}n)$ by the union bound. As $\max_{||x||_2=1}||\tilde Ax||_2\,\leqslant\max_{k,l}(|\tilde A_{kl}|)p$ we have
\begin{align*}
 P[\sigma_{max}(\tilde A)\leqslant\lambda_0/2]&= P[\max_{||x||_2=1}||\tilde Ax||_2\leqslant\lambda_0/2]\\
 &\geqslant P[\max_{k,l}(|\tilde A_{kl}|)\leqslant\lambda_0/2p]\geqslant 1-p^2\exp(-C_{\lambda_0,\sigma_0,p}n)
\end{align*}

For $\lambda_{max}(\tilde B_n)$ we can apply Theorem \ref{matrixBernstein} to obtain that it is lower bounded by $\lambda_0/2$ with high probability. In fact, on $E$ the squared $L^2$ norm of $S_i$ is bounded by $p/\sigma_0^2$, and therefore $||S_iS_i^\top||=\lambda_{max}(S_iS_i^\top)\leqslant p/\sigma_0^2$ so that $||\maE[SS^\top]||\leqslant\maE[||S_i||^2]\leqslant p/\sigma_0^2$ by Jensen's inequality.  Moreover, for $\boldsymbol S_i:=\maE[SS^\top]-S_iS_i^\top$ we have $\maE[\boldsymbol S_i]=0$ and 
$$||\boldsymbol S_i||\leqslant\frac{1}{n}(||S_iS_i^\top||+||\maE[SS^\top]||)\leqslant\frac{2p}{n\sigma_0^2}\,.$$
Using the same argument in Section 1.6.3 of \cite{matrix_ineq} we find that $v(\Tilde{B}_n)\leqslant p^2/n\sigma_0^4$ so that on $E$,
\begin{align*}
P(||\tilde B_n||>\frac{\lambda_0}{2})&\leqslant P\left(||\tilde B_n||>\frac{\lambda_0}{2} \text{ on } E\right)+P(E^c)\\
&\leqslant 2p\exp\Bigg(\frac{-\lambda_0^2\sigma_0^4n/8}{p^2 + p\lambda_0\sigma^2_0/3}\Bigg)+p^2\exp(-C_{\lambda_0,\sigma_0,p}n).
\end{align*}
At the end we note that $$2\lambda_{0}\leqslant\lambda_{min}(Cor(X))=\lambda_{min}(\maE[(S-\overline{S})(S-\overline{S})^\top])\leqslant\lambda_{min}(\maE[SS^\top]).$$
Summing up all the estimations, we know that by a union bound there exists a constant $C_{\lambda_0,\sigma_0,p}>0$ such that
$P[\lambda_{min}(\tilde {\boldsymbol X}^\top\tilde {\boldsymbol X})>\lambda_0]\geqslant 1-\exp(-C_{\lambda_0,\sigma_0,p}n)$. 
\end{proof}

\subsection{Proof of Theorem \ref{theorem:recursion2}}
We can now prove the recursion formula for the expectation of the excess error.\\

\begin{proof}
We first deal with the case where \textsc{con} is not included. We know that whatever non-\textsc{con} model is chosen on $T$, its impurity gain is always larger than that of \textsc{lin} (because the degrees of freedom of \textsc{lin} is the smallest). Moreover, the first truncation procedure does not deteriorate the gain of the model, since it only eliminates bad predictions. Therefore when $L^k_n-L^*_n$ is given and positive,  by Lemma \ref{l4}, with probability at least $\sum_{T\in \maT_k}\exp(-C_{\lambda_0,\sigma_0,p}t)$, it holds that
\allowdisplaybreaks 
\begin{align*}
L_n^{k+1}-L^*_n &= L_n^{k+1} - L_n^k + L_n^k-L^*_n = L_n^k-L^*_n - (L_n^{k+1} - L_n^k)\\
&\leqslant L_n^k-L^*_n-\sum_{T\in \maT_K}w(T)\Delta_{lin}(T)\\
&= L_n^k-L^*_n-\sum_{T\in \maT_K}w(T) (L^{k+1}_{lin,n}(T)-L^k_{n}(T))\\ 
&\overset{(i)}{=} L_n^k-L^*_n - \sum_{T\in \maT_K}w(T)\Big((\overline{Y}-\overline{X_T\beta_T})^2+\Big\langle Y-X_T\beta_T,\frac{X^{(j)}-\overline{X^{(j)}}}{\sigma_{j}}\Big\rangle_t^2\Big)\\
&\overset{(ii)}{=} L_n^k-L^*_n - \sum_{T\in \maT_K}w(T)(\bar r^k_T)^2-\sum_{T\in \maT_K}w(T)\Tilde{\Delta}_{lin}(T)\\
&\overset{(iii)}{\leqslant} L_n^k-L^*_n - \sum_{T\in \maT_K}w(T)(\bar r^k_T)^2\\
&\quad -\sum_{T\in \maT_K}\frac{\lambda_0}{4p}w(T)(||r^k_T-\bar r^k_T||^2_t-L^*_n(T))\\
&\overset{(iv)}{\leqslant} L_n^k-L^*_n-\sum_{T\in \maT_K}\frac{\lambda_0}{4p}w(T)(L^k_n(T)-L^*_n(T))\\
&\overset{(v)}{\leqslant} L_n^k-L^*_n-\frac{\lambda_0}{4p}(L^k_n-L^*_n)\\
&=(L_n^k-L^*_n)\Big(1-\frac{\lambda_0}{4p}\Big)
\end{align*}
where $r^k_T:=Y-X_T\beta_T$ are the residuals on $T$. Here, (i) follows from equation (\ref{lin_pre}). The $\Tilde{\Delta}_{lin}(T)$ in (ii) means the gain of \textsc{lin} after regressing out the mean. The last equation in (iii) follows from Lemma \ref{l4}.
(iv) is due to $||r^k_T-\bar r^k_T||^2_t+||\bar r^k_T||^2_t=||r^k_T||^2_t$ and the fact that the coefficient $\frac{\lambda_0}{4p}$ is less than 1.
(v) follows from the fact that the weighted sum of the optimal loss in each leaf node $\sum_{T\in\maT_k}w(T)L^*_{n}(T)$ is less than $L^*_n$ for each $k$. As argued in the previous lemma, the factor $1-\frac{\lambda_0}{4p}$ is always positive.

It remains to control the `bad' probability, i.e. the case where $\lambda_{min}(T)<\lambda_0$. We have that
$$P(\exists T\in \maT_K\text{ s.t. }\lambda_{min}(T)<\lambda_0)\leqslant\sum_{T\in \cup_{k=1}^K\maT_k}\exp(-C_{\lambda_0,\sigma_0,p}t)\leqslant K\sum_{T\in \maT_K}\exp(-C_{\lambda_0,\sigma_0,p}t).$$
By condition 1 it holds that
\begin{align*}
K\sum_{T\in \maT_K}\exp(-C_{\lambda_0,\sigma_0,p}t)&\leqslant Kn^{1-\alpha}\exp(-C_{\lambda_0,\sigma_0,p}n^{\alpha})\\
&=\exp(-C_{\lambda_0,\sigma_0,p}n^{\alpha}+\log K+(1-\alpha)\log n)\\ 
&\precsim\exp(-C'_{\lambda_0,\sigma_0,p}n^{\alpha})
\end{align*}
as long as $\log K\precsim n^\alpha$ which is satisfied by choosing $K$ of order $\log n$. To sum up everything, if we let $G:=\{\forall T\in \maT_K, \lambda_
{pmin}(T)>\lambda_0\}$ and assume $\maE[\epsilon^4]<\infty$, we have for sufficiently large $n$,
\begin{align*}
\maE[L^k_n-L^*_n]&=\maE_X[\mathbbm 1_{G}\maE_{\epsilon|X}[L^k_n-L^*_n]+\mathbbm 1_{G^c}\maE_{\epsilon|X}[L^k_n-L^*_n]]\\
&\leqslant\sqrt{P(X\in G)}\sqrt{\maE_{X}[(\maE_{\epsilon|X\in G}[L^k_n-L^*_n])^2]}\\
&\quad +\sqrt{P(X\in G^c)}\sqrt{\maE_{X}[(\maE_{\epsilon|X\in G^c}[L^0_n-L^*_n])^2]}\\
&\leqslant  (1-\exp(-C'_{\lambda_0,\sigma_0,p}n^{\alpha}))^{1/2}\Big(1-\frac{\lambda_0}{4p}\Big)^{k}\sqrt{\maE[(L^0_n-L^*_n)^2]}\\
&\quad + \exp(-C''_{\lambda_0,\sigma_0,p}n^{\alpha})\sqrt{\maE[(L^0_n-L^*_n)^2]}\\
&\leqslant\Big(1-\frac{\lambda_0}{4p}\Big)^{k}\sqrt {\maE[(L^0_n-L^*_n)^2]}+\mathcal{O}(1/n).
\end{align*}
Since we have an exponential decay in the probability, the error can be bounded by $\mathcal O(1/n)$ if we let $K\asymp\log n$.

When \textsc{con} is included, we can derive an analog of (\ref{conbound}) by using Lemma \ref{l4} and Proposition \ref{prop:bic} to get for all $T$ with \textsc{con}:
$$P_{X,\epsilon}\Bigg[\lambda_0\frac{L_n(T)-L^*_n(T)}{4pR_0}<\frac{1}{t}\Big(1-\exp\Big((1-v_{lin})\frac{\log t}{t}\Big)\Big)\Bigg]\geqslant 1-\exp(C_{\lambda_0,\sigma_0,p}n).$$
Moreover, since the error term has a finite fourth moment and the true underlying $f$ is linear, we have with probability at least $1-\exp(C_{\lambda_0,\sigma_0,p}n)$ that
$$\maE_{\epsilon|X}[L_n(T)-L^*_n(T)]\precsim \maE_{\epsilon|X}[Y^2]{\Big(1-\exp\Big((1-v_{lin})\frac{\log t}{t}\Big)\Big)}\precsim{\frac{(v_{lin}-1)\log t}{t}}.$$
Let us denote by $\maT_K^{\textsc{con}}$ all the \textsc{con} nodes in the tree up to depth $K$. Then the expectation of the weighted sum of the errors on all of these nodes satisfies
\begin{align*}
\maE_{X,\epsilon}\Big[\sum_{T\in \maT_K^{\textsc{con}}}w(T)(L_n(T)-L^*_n(T))\Big]&\leqslant
\maE_X\Big[\mathbbm 1_{G^c}\sum_{T\in \maT_K^{\textsc{con}}}\maE_{\epsilon|X}[w(T)(L_n(T)-L^*_n(T))]\Big]\\
&+\maE_X\Big[\mathbbm 1_{G}\sum_{T\in \maT_K^{\textsc{con}}}\frac{(v_{lin}-1)\log(t)}{n}\Big]\\
&\precsim \mathcal{O}\Big(\frac{1}{n}\Big) + \mathcal{O}\Big(\frac{N_{leaves}\log(n)}{n}\Big).
\end{align*}
Noticing that in our previous estimation for non-\textsc{con} nodes, the estimation of  $L^{k+1}_n(T)-L^*_n(T)$ is always positive if $L^k_n(T)-L^*_n(T)>0$ due to $\lambda_0\leqslant p$ (see the proof of Lemma \ref{l4}), we were able to first ignore the \textsc{con} nodes to get the recursion formula, and at the end to add back the errors in all \textsc{con} nodes to get the final results.
\end{proof}

\subsection{Proof of Corollary \ref{cor:rate}}
Finally, by applying the preceding results and the proof of the consistency theorem, we get the convergence rate of PILOT on linear model data.\\

\begin{proof}
We know that $\maE[L^*_n-||\bX\beta||^2_n]=\mathcal O(1/n)$ for the true parameter $\beta$. So if we choose $K_n= log_\gamma(n)$, by the preceding theorem we have $\maE[L^{K_n}_n-||\bX\beta||^2_n]=\mathcal O(N_{leaves}\log(n)/n)$. Finally we can follow the proof of Theorem \ref{theorem:consistency} in which we replace $2^K$ by $N_{leaves}$ and use Theorem \ref{theorem:recursion2} to derive a similar oracle inequality as (\ref{oracle}):
\begin{align*}
\maE[||\boldsymbol X\beta-\hat f(\mathcal{T}_K)||^2]&\precsim \mathcal O\Big(\frac{N_{leaves}\log(n)}{n}\Big)+2\maE[|L^{K_n}_n-||\bX\beta||^2_n|]\\
&\precsim\mathcal{O}\Big(\frac{\log(n)}{n^{\alpha}}\Big).  
\end{align*}
\end{proof}

As a side remark, in the situation where some predictors are correlated with each other, in practice \textsc{plin} and \textsc{pcon} can lead to even faster convergence than pure $L^2$ boosting. 
\end{document}